\documentclass[a4paper]{article}
\usepackage[utf8x]{inputenc}

\usepackage[a4paper,top=3cm,bottom=2cm,left=3cm,right=3cm,marginparwidth=1.75cm]{geometry}
\usepackage{natbib}
\usepackage{xspace}
\usepackage{algorithm}
\usepackage{algorithmic}
\usepackage{threeparttable}
\usepackage{amsmath,amsfonts,amsthm}
\usepackage{url}
\usepackage{multirow}
\usepackage{subfig}
\usepackage{graphicx}
\usepackage{tikz}
\usepackage{booktabs}
\usepackage{pgfplots}
\usepackage{filecontents}
\usepackage{pgfplotstable}

\usepackage{lscape}

\usetikzlibrary{calc}
\usetikzlibrary{mindmap,trees}
\usetikzlibrary{arrows}
\usetikzlibrary{arrows.meta}
\usetikzlibrary{shadows}
\usetikzlibrary{plotmarks}
\usetikzlibrary{backgrounds}
\usetikzlibrary{shapes,arrows}
\usetikzlibrary{shapes.symbols}
\usetikzlibrary{shapes.callouts}

\newcommand{\twomax}{\textsc{Two\-Max}\xspace}
\newcommand{\onemax}{\textsc{One\-Max}\xspace}
\newcommand{\zeromax}{\textsc{Zero\-Max}\xspace}

\newcommand{\poly}[1]{\mathrm{poly}(#1)}
\newcommand{\di}{\mathrm{d}}
\newcommand{\dist}[2]{\di\mathord{\left(#1,#2\right)}}
\newcommand{\hamm}{\mathrm{H}}
\newcommand{\Hamm}[2]{\hamm\mathord{\left(#1,#2\right)}}
\newcommand{\ones}[1]{\left|#1\right|_1}

\newcommand{\Real}{\mathbb{R}\xspace}

\newcommand{\muea}{{\upshape{}($\mu$+1)~EA}\xspace}
\newcommand{\eaoneone}{{\upshape{}(1+1)~EA}\xspace}

\newcommand{\Bigo}[1]{\mathrm{O}\mathord{\left(#1\right)}\xspace}
\newcommand{\Smallo}[1]{\mathrm{o}\mathord{\left(#1\right)}\xspace}
\newcommand{\Bigom}[1]{\Omega\mathord{\left(#1\right)}\xspace}
\newcommand{\Smallom}[1]{\omega\mathord{\left(#1\right)}\xspace}
\newcommand{\Bigth}[1]{\Theta\mathord{\left(#1\right)}\xspace}

\newtheorem{theorem}{Theorem}
\newtheorem{lemma}[theorem]{Lemma}

\newtheorem{definition}[theorem]{Definition}

\newcommand{\ie}{i.\,e.,\xspace}
\newcommand{\eg}{e.\,g.,\xspace}

\newcommand{\e}{\mathrm{E}}
\newcommand{\E}[1]{\e\mathord{\left(#1\right)}}
\newcommand{\prob}{\mathrm{Pr}}
\newcommand{\Prob}[1]{\prob\mathord{\left(#1\right)}}
\newcommand{\Bin}[1]{\mathrm{Bin}\mathord{\left(#1\right)}\xspace}

\newcommand{\IFTHEN}[2]{\STATE \algorithmicif\ #1\ \algorithmicthen\ #2\ }

\newcommand{\ignore}[1]{}

\newcommand{\EA}{\eaoneone}
\newcommand{\lone}{x}

\pgfplotsset{compat=newest}
\pgfplotsset{
	enlargelimits=true,
  	grid=major,
  	scale only axis,
}

\usepgfplotslibrary{statistics}
\usepgfplotslibrary{groupplots}
\usepgfplotslibrary{fillbetween}

\pgfmathdeclarefunction{fpumod}{2}{%
    \pgfmathfloatdivide{#1}{#2}%
    \pgfmathfloatint{\pgfmathresult}%
    \pgfmathfloatmultiply{\pgfmathresult}{#2}%
    \pgfmathfloatsubtract{#1}{\pgfmathresult}%
    \pgfmathfloatifapproxequalrel{\pgfmathresult}{#2}{\def\pgfmathresult{3}}{}%
}

\pgfplotstableread[col sep = tab]{resultfiles/RTS_WR_OBOUND/rtsg.txt}\successrtsg
\pgfplotstableread[col sep = tab]{resultfiles/RTS_WR_OBOUND/rtsp.txt}\successrtsp

\pgfplotstableread[col sep = tab]{resultfiles/RTS_NR_OBOUND/rtsgnr.txt}\successrtsgnr
\pgfplotstableread[col sep = tab]{resultfiles/RTS_NR_OBOUND/rtspnr.txt}\successrtspnr

\pgfplotstableread[col sep = tab]{resultfiles/RTS_WR_NBOUND/rts_wr_g.txt}\successrtsgwrnt
\pgfplotstableread[col sep = tab]{resultfiles/RTS_WR_NBOUND/rts_wr_p.txt}\successrtspwrnt

\pgfplotstableread[col sep = tab]{resultfiles/RTS_NR_NBOUND/rtsgnrnt.txt}\successrtsgnrnt
\pgfplotstableread[col sep = tab]{resultfiles/RTS_NR_NBOUND/rtspnrnt.txt}\successrtspnrnt

\pgfplotstableread[col sep = comma]{resultfiles/TAKEOVER_BEST/rts_wrg_tb_mean_std.txt}\wrgmeanstd

\pgfplotstableread[col sep = comma]{resultfiles/TAKEOVER_BEST/rts_nrg_tb_mean_std.txt}\wtrgmeanstd

\pgfplotstableread[col sep = comma]{resultfiles/TIME_BOUND/rts_nr_twmu.txt}\rtsnrtwmu

\renewcommand{\IFTHEN}[2]{\STATE \algorithmicif\ #1\ \algorithmicthen\ #2\ \algorithmicendif}
\renewcommand{\E}[1]{\e\mathord{\left[#1\right]}}
\newcommand{\harm}[1]{\mathrm{H}_{#1}}

\begin{document}


\title{\bf Runtime Analysis of Restricted Tournament Selection for Bimodal Optimisation}

\author{Edgar Covantes Osuna \\ School of Engineering and Sciences, Tecnol\'{o}gico de Monterrey,\\ Nuevo Le\'{o}n, M\'{e}xico.\\
\and 
Dirk Sudholt\\
Department of Computer Science, University of Sheffield, 
        United Kingdom and\\
        Chair of Algorithms for Intelligent Systems, University of Passau, Germany.\\ 
}

\maketitle

\begin{abstract}
Niching methods have been developed to maintain the population diversity, to investigate many peaks in parallel and to reduce the effect of genetic drift.
We present the first rigorous runtime analyses of restricted tournament selection (RTS), embedded in a \muea, and analyse its effectiveness at finding both optima of the bimodal function \twomax. In RTS, an offspring competes against the closest individual, with respect to some distance measure, amongst~$w$ (window size) population members (chosen uniformly at random with replacement), to encourage competition within the same niche. We prove that RTS finds both optima on \twomax efficiently if the window size~$w$ is large enough. However, if~$w$ is too small, RTS fails to find both optima even in exponential time, with high probability. 
We further consider a variant of RTS selecting individuals for the tournament \emph{without} replacement.
It yields a more diverse tournament and is more effective at preventing one niche from taking over the other. However, this comes at the expense of a slower progress towards optima when a niche collapses to a single individual. Our theoretical results are accompanied by experimental studies that shed light on parameters not covered by the theoretical results and support a conjectured lower runtime bound. 
\end{abstract}





\section{Introduction}
\label{sec:intro}

One of the major challenges when using Evolutionary Algorithms (EAs) is to maintain the diversity in the population in order to prevent premature convergence.
One way of maintaining diversity is to use niching methods, which are based on the mechanics of natural ecosystems~\citep{Shir2012}. Niches can be viewed as subspaces in the environment that can support different types of life. A specie is defined as a group of individuals with similar features, capable of interbreeding among themselves, but unable to breed with individuals outside their group. Species can be defined as similar individuals of a specific niche in terms of similarity metrics. In evolutionary algorithms the term niche is used for the search space domain, and species for the set of individuals with similar characteristics. 

Niching methods have been developed to reduce the effect of genetic drift resulting from the selection operator in standard EAs, to maintain the population diversity, and to allow the EA to investigate many peaks simultaneously, thus avoiding getting trapped in local optima before the fitness landscape is explored properly~\citep{Sareni1998}. This is often done by modifying the selection process of individuals, taking into account not only the value of the fitness function but also the distribution of individuals in the space of genotypes or phenotypes~\citep{Glibovets2013}. 
Many niching techniques have been introduced to solve problems where it is necessary to identify multiple optima (multimodal problems), either local or global optima. The main goal for these techniques is to form and maintain multiple, diverse, final solutions for an exponential to infinite time period with respect to population size, whether these solutions are of identical fitness or of varying fitness~\citep{Shir2012,Crepinsek2013,Glibovets2013,Squillero2016}.
Given such a variety of mechanisms to choose from, it is often not clear which mechanism is the best choice for a particular problem.

Most of the analyses and comparisons made between niching methods used empirical investigations on benchmark functions~\citep{Sareni1998,Singh2006}. Theoretical runtime analyses have been performed that rigorously quantify the expected time needed to find one or several global optima~\citep{Friedrich2009,Oliveto2019,Covantes2019b,Covantes2019}. One example where theoretical results are used to inform the choice of the EA parameters' values in an empirical performance study of EAs with diversity mechanisms can be found in~\cite{Covantes2018b}. Both approaches are important to understand how these mechanisms impact the performance of EAs and whether and how they enhance the search for good individuals. These different approaches can help to explain when a niching mechanism should be used, which niching mechanism works best, and how to set parameters. 

Previous theoretical studies~\citep{Friedrich2009,Oliveto2019,Covantes2019b,Covantes2019} compared the expected running time of different diversity mechanisms when embedded in a simple baseline EA, the \muea. All mechanisms were considered on the well-known bimodal function $\twomax(x):=\max\left\{n-\sum_{i=1}^{n}x_i,\sum_{i=1}^{n}x_i\right\}$. \twomax consists of two different symmetric slopes (or branches) \zeromax and \onemax with $0^n$ and $1^n$ as global optima, respectively, and the goal is to evolve a population that contains both optima\footnote{In \cite{Friedrich2009} an additional fitness value for $1^n$ was added to distinguish between a local optimum $0^n$ and a unique global optimum. There the goal was to find the global optimum, and all approaches had a baseline probability of $1/2$ of climbing up the right branch by chance. We use the same approach as~\cite{Oliveto2019,Covantes2019b}, and consider the original definition of \twomax and the goal of finding both global optima. The discussion and presentation of previous work from~\cite{Friedrich2009} is adapted to our setting. We refer to~\cite{Sudholt2020} for details.}.

\twomax was chosen because it is simply structured, hence facilitating a theoretical analysis, and it is hard for EAs to find both optima as they have the maximum possible Hamming distance. The results allowed for a fair comparison across a wide range of diversity mechanisms, revealing that some mechanisms like fitness diversity or avoiding genotype duplicates and probabilistic crowding perform badly, while other mechanisms like fitness sharing, clearing or deterministic crowding perform surprisingly well (see Table~\ref{tab:divmech} and Section~\ref{sec:previous-work}). 

We contribute to this line of work by studying the performance of the crowding mechanism called \emph{restricted tournament selection (RTS)}. This mechanism is a well-known technique as covered in tutorials and surveys for diversity-preserving mechanisms~\citep{Shir2012,Crepinsek2013,Glibovets2013,Squillero2016} and compared in empirical investigations \citep{Sareni1998,Singh2006,Covantes2018b}. 

Restricted tournament selection is a modification of the classical tournament selection for multimodal optimisation that exhibits niching capabilities. For each offspring, RTS sets up a tournament of $w$ (window size) individuals, chosen uniformly at random with replacement\footnote{We believe that the original definition of RTS picks individuals \emph{with} replacement, that is, it is possible to select multiple copies of one individual, ending up with fewer than $w$ different genotypes. The exact implementation was not explained in~\cite{Harik1995}, however most calculations in~\cite{Harik1995} assume selection with replacement. Oddly enough, from all cited papers here, where RTS has been analysed or used, only \cite{Garcia2012} make explicit mention of the selection policy used to select the $w$ individuals (uniformly at random with replacement). We consider a variant \emph{without} replacement in Section~\ref{sec:rts_nor}.} from the population. The offspring competes against the closest individual with respect to some distance measure from the tournament and the best individual is selected for the next generation. This form of tournament restricts an entering individual from competing with others too different from it \citep{Harik1995}. RTS has been analysed empirically for the classical comparison between crowding mechanisms for multimodal optimisation as a replacement strategy~\citep{Qu2010,Sareni1998,Singh2006}. Recent applications for engineering problems with multimodal domains include facility layout design~\citep{Garcia2015} and the design of product lines~\citep{Tsafarakis2016} with reported better results compared to the other variants without RTS. 

However, we are lacking a good understanding of when and why it performs well and how it compares to diversity mechanisms analysed previously. Our contribution is to provide a rigorous theoretical runtime analysis accompanied by experimental studies for this mechanism in the context of the \muea on \twomax, to rigorously assess its performance in comparison to other diversity mechanisms. In addition, our goal is to provide insights into the working principles of this mechanism to narrow the gap between theory and practice, and to enhance our understanding of its strengths and weaknesses.

\subsection{Our contribution}

For the \muea with RTS, we show in Section~\ref{sec:rtswitrep} that the mechanism succeeds in finding both optima of \twomax in the same way as deterministic crowding, provided that the window size~$w$ is chosen large enough, in  time $\Bigo{\mu n \log{n}}$ with high probability (Section~\ref{sec:large-w}). 
We also show that, if the window size is too small, then it cannot prevent one branch taking over the other, leading to exponential running times with high probability (Section~\ref{sec:small-w}).

We further consider a variant of RTS, where the tournament chooses individuals \emph{without replacement}. This simple change tends to make the tournament more diverse. We investigate its effect on the performance of the \muea in Section~\ref{sec:rts_nor}. We show that the positive result for RTS with replacement holds for its variant without replacement for a weaker condition $w\geq \mu$ (Section~\ref{sec:large_w_rts_nor}). For small window sizes, the situation is less clear as the previous analysis for RTS with replacement breaks down. We show that RTS without replacement will find both optima with certainty once two subpopulations have evolved whose best fitness is above a certain threshold value. This means that takeover cannot happen during the final stages of a run where both branches are being explored. However, during early stages of a run, takeover can still happen under rare conditions. 

We also find that, surprisingly, in runs where both optima are found, RTS without replacement seems to take significantly more time. This is because typically a population will evolve until only a single individual remains on one branch, and this individual takes time $\Theta((\mu^2/w) \cdot n \log n)$ to evolve a global optimum, under certain assumptions (see Section~\ref{sec:small_w_rts_nor}). This is by a factor of order $\mu/w$ larger than the upper bound of $\Bigo{\mu n \log n}$ for RTS with replacement from Section~\ref{sec:large-w}.

Our theoretical results are accompanied by experimental studies that match the theoretical results and also shed light on parameters not covered by the theoretical results. We performed experiments for RTS with replacement and its variant without replacement in order to observe different aspects of the algorithms such as their ability to find both optima on \twomax (Sections~\ref{sec:well-known-time-budget} and \ref{sec:extra_time_budget}). We further support the conjectured lower bound of $\Bigom{(\mu^2/w) \cdot  n \log n}$ for small values of~$w$ by measuring the time needed for RTS without replacement to find both optima on \twomax (Section~\ref{sec:runtime_growth}). Finally, we assess when takeover is more likely to happen and which variant is more resilient to takeover (Section~\ref{sec:exp_takeover}).

This article significantly extends a preliminary conference paper~\citep{Covantes2018a} that contained preliminary theoretical results for the original RTS (selecting the tournament with replacement) and preliminary experimental results. In this manuscript, the negative result for RTS with small~$w$ in \cite{Covantes2018a} has been improved.

\section{Previous Work and Preliminaries}
\label{sec:previous-work}

There has been a line of work comparing various diversity mechanisms on \twomax in the context of the simple \muea (see Algorithm~\ref{alg:muea}). The \muea starts with a population of size~$\mu$ created uniformly at random and generates one offspring via mutation; the resulting offspring competes with an individual selected uniformly at random from the subpopulation with worst fitness and the best individual replaces the worst (in case of ties, the offspring is preferred). Table~\ref{tab:divmech} summarises all known results, including our contributions (shown in bold) and conditions involving population size~$\mu$ and specific parameters of each diversity mechanism explained below. Results from~\cite{Friedrich2009} are adapted to our definition of \twomax; see~\cite{Sudholt2020} for details.

\begin{algorithm}[!ht]
  \begin{algorithmic}[1]
  	\STATE Initialise $P$ with $\mu$ individuals chosen uniformly at random
    \WHILE{optimum \NOT found} 
    	\STATE{Choose $x \in P$ uniformly at random} 
    	\STATE{Create $y$ by flipping each bit in $x$ independently with probability $1/n$.}
        \STATE{Choose $z \in P$ uniformly at random from all individuals with worst fitness in $P$.}
        \IFTHEN{$f(y) \geq f(z)$}{$P = P \setminus \{z\}\cup \{y\}$}
    \ENDWHILE
  \end{algorithmic}
  \caption{\muea}
  \label{alg:muea}
\end{algorithm}

\begin{table}[!ht]
	\centering
	\caption{Overview of runtime analyses for the \muea with diversity mechanisms on \twomax, showing the probability of finding both optima within (expected) time $\Bigo{\mu n \log n}$. Results derived in this paper are shown in bold.}
    \label{tab:divmech}
    \begin{threeparttable}
    \begin{tabular}{lcl}
      \hline
      \bf Diversity Mechanism&\bf Success prob. &\bf Conditions\\
      \hline
      Plain \muea \tnote{1} & $\Smallo{1}$& $\mu=\Smallo{n/\log{n}}$\\
      \midrule
      No Duplicates \tnote{1} & & \\
      \quad Genotype &$\Smallo{1}$&$\mu=\Smallo{\sqrt{n}}$\\
      \quad Fitness &$\Smallo{1}$&$\mu=\poly{n}$\\
      \midrule
      Deterministic Crowding \tnote{1} &$1-2^{-\mu+1}$&all $\mu$\\
      \midrule
      Fitness Sharing ($\sigma=n/2$) & & \\
      \quad Population-based \tnote{1} &$1$&$\mu\geq 2$\\
      \quad Individual-based \tnote{2} &$1$&$\mu\geq 3$\\
      \midrule
      Clearing ($\sigma=n/2$) \tnote{3} &$1$&$\mu\geq\kappa n^2$\\
      \midrule
      Probabilistic Crowding \tnote{4}&$2^{-\Bigom{n}}$& all $\mu$\\
      \midrule
      Probabilistic Crowding with Scaling \tnote{4} & &\\
      \quad General bases $\alpha$ & $2^{-\Bigom{n/\alpha}}$ & all $\alpha \ge 1$\\
      \quad Very large $\alpha$ & $1-2^{-\mu+1}$ & $\alpha \ge (1\!+\!\Bigom{1})en$\\
      \midrule
      Generalised Crowding \tnote{4} & &\\
      \quad General scaling factors $\phi$ & $2^{-\Bigom{\phi n}}$ & all $\phi \le 1$\\
      \quad Very small $\phi$ & $1-2^{-\mu+1}$ & $\phi \le \frac{1 - \Bigom{1}}{e^2n}$\\
      \midrule
      \textbf{RTS with replacement}&&\\
      \quad \textbf{Small window size \boldmath$w$ (Theorem~\ref{the:badresrts})}&\boldmath$\Smallo{1}$&\boldmath$\mu=\Smallo{n^{1/w}}$\\
      \quad \textbf{Large window size \boldmath$w$ (Theorem~\ref{the:positive-result-for-rts})}&\boldmath$1-2^{-\mu'+3}$&\boldmath $w \ge 2.5\mu \ln{n}$\\
      \midrule
      \textbf{RTS without replacement}&&\\
      \quad \textbf{Large window size \boldmath$w$ (Theorem~\ref{the:positive-result-for-rts-without-replacement})}&\boldmath$1-2^{-\mu'+3}$&\boldmath $w \ge \mu$\\
      \bottomrule
    \end{tabular}
    \begin{tablenotes}
    \footnotesize \item[1] \cite{Friedrich2009}.
    \footnotesize \item[2] \cite{Oliveto2019}.
    \footnotesize \item[3] \cite{Covantes2019b}.
    \footnotesize \item[4] \cite{Covantes2019}.
\end{tablenotes}
    \end{threeparttable}
\end{table}

The notion of success in previously studied mechanisms was being able to find both optima in (expected) time $\Bigo{\mu n \log n}$.
To put this time bound in context, it is easy to see that the simple \EA finds one optimum of \twomax in expected time $\Bigth{n \log n}$. The time bound $\Bigo{\mu n \log n}$ includes an additional factor of~$\mu$ to account for the overhead of evolving $\mu$ individuals instead of one. This overhead is necessary for mechanisms like deterministic crowding as used in~\cite{Friedrich2009}, which essentially evolves $\mu$ independent lineages of \EA{}s~\citep{Friedrich2009}.

Table~\ref{tab:divmech} shows that not all mechanisms succeed in finding both optima on \twomax in (expected) time $\Bigo{\mu n \log n}$. \cite{Friedrich2009} showed that the plain \muea and the simple mechanisms like \emph{avoiding genotype} or \emph{fitness duplicates} are not able to prevent the extinction of one branch, ending with the population converging to one optimum, with high probability. \emph{Deterministic crowding} with a sufficiently large population is able to reach both optima with probability $1-2^{-\mu+1}$ in expected time $\Bigo{\mu n \log n}$. This probability converges to~1 exponentially fast in~$\mu$; for instance, a small population size of $\mu=10$ already gives a success probability of $\approx\! 0.998$ and for $\mu=30$ it grows to $\approx\! 0.9999999981$.
A \emph{population-based fitness sharing} approach, constructing the best possible new population amongst parents and offspring, with $\mu\geq 2$ and a sharing radius of $\sigma=n/2$ is able to find both optima in expected runtime $\Bigo{\mu n \log n}$. The drawback of this approach is that all possible size $\mu$ subsets of this union of size $\mu+\lambda$ (where $\lambda$ is the offspring population size) need to be examined. This is prohibitive for large $\mu$ and~$\lambda$.

\cite{Oliveto2019} studied the original \emph{fitness sharing} approach and showed that a population size $\mu=2$ is not sufficient to find both optima in polynomial time; the success probability is only $1/2-\Bigom{1}$. However, with $\mu\geq 3$ fitness sharing again finds both optima in expected time $\Bigo{\mu n\log{n}}$. \cite{Covantes2019b} analysed the \emph{clearing} mechanism and showed that it can optimise all functions of unitation---function defined over the number of 1-bits contained in a string---in expected time $\Bigo{\mu n\log{n}}$ when the distance function and parameters like the clearing radius $\sigma$, the niche capacity $\kappa$ (how many winners a niche can support) and $\mu$ are chosen appropriately. In the case of large niches, that is, with a clearing radius of $\sigma=n/2$, it is able to find both optima in expected time $\Bigo{\mu n\log{n}}$. 

Finally, \cite{Covantes2019} analysed \emph{probabilistic crowding}~\citep{Mengsheol1999} (extending preliminary results from~\cite{Covantes2018a}, which are not included in this manuscript) and showed that it requires exponential time with overwhelming probability on \twomax and general classes of functions with bounded gradients. They showed that probabilistic crowding is unable to evolve solutions that are significantly closer to any global optimum than those found by random search, even when given exponential time $2^{\Bigom{n}}$ for all population sizes.
Probabilistic crowding requires exponential time even when applying exponential scaling to \twomax: for every constant base $\alpha$, and even values up to $\alpha=\Bigo{n^{1-\varepsilon}}$, on the function $\alpha^{\twomax(x)}$, the selection pressure is still too low, leading to an exponential time $2^{\Bigom{n/\alpha}}$ with overwhelming probability. Only when $\alpha=\Bigom{n}$, the selection pressure becomes large enough to enable hill climbing. In this case, probabilistic crowding with scaling is as successful on \twomax as deterministic crowding.

Generalised crowding~\citep{Galan2010} is a variant that generalises both deterministic and probabilistic crowding through the choice of a parameter called \emph{scaling factor $\phi\in [0,1]$} that diminishes the impact of the inferior search point. With $\phi = 1$ we have probabilistic crowding and $\phi=0$ yields deterministic crowding. \cite{Covantes2019} showed that if $\phi = \Bigom{n^{-1+\varepsilon}}$, for a constant $\varepsilon > 0$, then this gives exponential time with overwhelming probability on all functions with bounded gradients. But if $\phi = \Bigo{1/n}$, generalised crowding behaves similarly to deterministic crowding and becomes effective on \twomax.

The above works did not consider crossover as recombining individuals from different branches is likely to create poor offspring. We therefore consider a \muea using mutation only and we introduce RTS into the definition of the \muea. In RTS a new offspring competes with the closest element with respect to some distance measure from $w$ (\emph{window size}) members selected uniformly at random, with replacement, from the population, and the better individual from this competition is selected. The \muea with RTS (Algorithm~\ref{alg:muearestouse}) is defined similarly as the plain \muea (see Algorithm~\ref{alg:muea}), the \muea with deterministic crowding in \cite{Friedrich2009} and the \muea with probabilistic crowding in \cite{Covantes2018a} to facilitate comparisons between all the algorithms and all available results.

In Algorithm~\ref{alg:muearestouse} an individual~$x$ is selected uniformly at random as a parent and a new individual~$y$ is created in the mutation step. Since we are not considering crossover and only one individual is created, $w$ individuals are selected uniformly at random with replacement and stored in a temporary population~$P^{*}$. Then in Line~\ref{alg:rts:ln:dis} an individual~$z$ is selected from population~$P^{*}$ with the minimum distance from~$y$ (ties are broken uniformly at random), and if the individual~$y$ has a fitness at least as good as~$z$, $y$ replaces~$z$.

\begin{algorithm}[!ht]
  \begin{algorithmic}[1]
    \STATE{Initialise $P$ with $\mu$ individuals chosen uniformly at random}
    \WHILE{stopping criterion \NOT met}
    	\STATE{Choose $x \in P$ uniformly at random}
        \STATE{Create $y$ by flipping bits in $x$ independently with probability~$1/n$.}
        \STATE{Select~$w$ individuals uniformly at random, with replacement, from~$P$ and store them in~$P^{*}$.}
        \STATE{Choose $z\in P^{*}$ with the minimum distance to $y$.}\label{alg:rts:ln:dis}
        \IFTHEN{$f(y) \geq f(z)$}{$P = P \setminus \{z\} \cup \{y\}$}
    \ENDWHILE
  \end{algorithmic}
  \caption{\muea with restricted tournament selection}
  \label{alg:muearestouse}
\end{algorithm}

As distance functions~$\dist{\cdot}{\cdot}$ we consider genotypic or Hamming distance, defined as the number of bits that have different values in $x$ and $y$: $\dist{x}{y}:=\Hamm{x}{y}:=\sum_{i=0}^{n-1}|x_i - y_i|$, and phenotypic distances as in~\cite{Friedrich2009,Oliveto2019,Covantes2019b} based on the number of ones: $\dist{x}{y}:=|\ones{x}-\ones{y}|$ where $\ones{x}$ and $\ones{y}$ denote the number of 1-bits in individual $x$ and $y$, respectively.

\subsection{Notation}
\label{sec:notation}
Our notion of time is defined as the number of function evaluations before the \muea achieves a stated goal such as finding a global optimum or finding both optima of \twomax. Since the \muea is initialised with $\mu$ individuals, and subsequently generates one offspring in each generation, the number of function evaluations is equal to $\mu$ plus the number of generations needed to achieve the set goal. The additional term of $\mu$ is only relevant for unreasonably large population sizes and is being tacitly ignored when it is absorbed in a runtime bound (such as $\Bigo{\mu n \log n}$) anyway.

We say that a function $f$ is exponential if $f = 2^{\Bigom{n^\varepsilon}}$ for a positive constant~$\varepsilon > 0$. A function $f$ is exponentially small if and only if $1/f$ is exponential. An event $A$ occurs with overwhelming probability if $1-\Prob{A}$ is exponentially small.

\subsection{Drift Theorems}

Our analysis will make heavy use of a technique called \emph{drift analysis}. In a nutshell, the progress of the algorithm is measured by a potential function such as the Hamming distance to an optimum where a potential of 0 indicates that an optimum has been found. The \emph{drift} is then defined as the expected change of this potential in one generation. 

The following \emph{multiplicative drift theorem} gives an upper bound on the expected time until the potential reaches 0 and an optimum has been found. It requires that the drift is at least proportional to its current state. It also gives a tail bound showing that the probability of exceeding this time is very small.

\begin{theorem}[Multiplicative drift theorem with tail bounds, adapted from \citealp{DoerrGoldberg2013}]
\label{the:multi_drift}
Let $\{X_t\}_{t\geq 0}$ be a sequence of random variables taking values in some set~$S$. Let $g:S\to\{0\}\cup\Real_{\geq 1}$ and assume that $g_{\max}:=\max\{g(x) \ | \ \allowbreak x\in S\}$ exists. Let $T:=\min\{t\geq~0:g(X_t)=0\}$. If there exists $\delta > 0$ such that
\[
\E{g(X_{t+1})\ | \ g(X_t)}\leq (1-\delta)g(X_t)
\]
then $\E{T}\leq (1+\ln g_{\max})/\delta$ and for every $c>0$, $\Prob{T>(\ln g_{\max} + c)/\delta}\leq e^{-c}$.
\end{theorem}

\section{Runtime Guarantees for Restricted Tournament Selection}
\label{sec:rtswitrep}

We first provide runtime guarantees for restricted tournament selection, showing under which conditions and parameter settings RTS is efficient and inefficient, respectively.

\subsection{Large Window Sizes Are Effective}
\label{sec:large-w}


We start off by proving that the time bound of $\Bigo{\mu n \log n}$, used as measure of success in all previous runtime analyses of diversity mechanisms (cf.\ Table~\ref{tab:divmech}), applies to many variants of the \muea. 
It assumes that the \muea never decreases the best fitness on a considered branch of \twomax; we will show in the proof of Theorem~\ref{the:positive-result-for-rts} that this assumption is met with high probability in the context of RTS with large window sizes. We also use this time bound to explain why $\Bigo{\mu n \log n}$ constitutes a natural benchmark in the context of the \muea enhanced with diversity mechanisms and to 
choose a sensible stopping criterion for our experimental analysis.

\begin{lemma}
\label{lem:time-mu-n-log-n}
Consider one branch of \twomax and a \muea with a replacement selection where the best fitness of all individuals on this branch never decreases and an offspring improving the current best fitness of its branch is always accepted\footnote{Compared to Lemma~3.3 in~\cite{Covantes2018a}, the statement about offspring was added for clarification; it was implicitly assumed and used in~\cite{Covantes2018a}. It is necessary as a branch might take longer to reach the optimum if better offspring are rejected; we will show later on, in Section~\ref{sec:rts-slowdown}, that this is the case for small window sizes~$w$. In Theorem~3.1 of~\cite{Covantes2018a}, the assumption was implicitly proven to hold for restricted tournament selection with large window sizes ($w \ge 2.5\mu \ln n)$.}. If the \muea is initialised with at least one individual on the branch then the optimum of the branch is found within time $\mu + 2e\mu n \ln n$ with probability $1-1/n$ and in expectation.
\end{lemma}
\begin{proof}
We apply the multiplicative drift theorem with tail bounds~\citep{DoerrGoldberg2010} (see Theorem~\ref{the:multi_drift}) to random variables $X_t$ that describe the Hamming distance of the closest individual to the targeted optimum. 
Note that $X_0 \le n/2$ as we start with an individual on the considered branch and the optimum has been found once {${X_t = 0}$}.

The probability of selecting an individual with Hamming distance $X_t$ is at least $1/\mu$. In order to create a better individual, it is sufficient that one of the $X_t$ differing bits is flipped and the other bits remain unchanged. Owing to our assumptions, such a better offspring will always survive. Each bit has a probability of being mutated of $1/n$ and the remaining bits remain unchanged with probability $(1-1/n)^{n-1} \ge 1/e$. Hence, the probability of creating an individual with a smaller Hamming distance is bounded as follows:
\[
\Prob{X_{t+1} < X_t \mid X_t} \ge \frac{1}{\mu}\cdot\frac{X_t}{n}\cdot\left(1-\frac{1}{n}\right)^{n-1}\geq\frac{X_t}{\mu en}.
\]
This implies
\[
\E{X_{t+1}\mid X_t} \le \left(1-\frac{1}{e\mu n}\right) X_t.
\]
Applying Theorem~\ref{the:multi_drift} with $\delta=\frac{1}{e\mu n}$, $g_{\max}=n$, and starting with an individual with Hamming distance $< n/2$ to the optimum, yields that the time till the optimum is found is at most $e\mu n \cdot (\ln(n/2) + \ln n) \le 2e\mu n \ln n$ with probability at most $1/n$ and in expectation. Adding a term of $\mu$ for the initial population completes the proof.
\end{proof}

Now we state the main result of this section, a positive result for RTS when the window size~$w$ is large. The following analysis shows that, if $w$ is chosen very large (even larger than the population size~$\mu$), the \muea with RTS behaves almost like the \muea with deterministic crowding.

\begin{theorem}
\label{the:positive-result-for-rts}
If $\mu = \Smallo{\sqrt{n}/\log n}$ and $w \ge 2.5\mu \ln n$ then the \muea with restricted tournament selection using genotypic or phenotypic distance finds both optima on \twomax in time $\Bigo{\mu n \log n}$ with probability at least ${1-2^{-\mu'+3}}$, where $\mu' := \min(\mu, \log n)$.
\end{theorem}

Note that, for small population sizes $\mu \le \log n$, the probability ${1-2^{-\mu'+3}}$ is close to the success probability $1-2^{-\mu+1}$ for deterministic crowding (see Table~\ref{tab:divmech} and~\citealp{Friedrich2009}), apart from a constant factor in front of the $2^{-\mu'}$ term. For both, the success rate converges to~1 very quickly for increasing population sizes. For larger population sizes, $\mu > \log n$, the probability bound for restricted tournament selection is capped at $1-2^{-\log n + 3} = 1-8/n$ as there is always a small probability of an unexpected takeover occurring.

In order to prove Theorem~\ref{the:positive-result-for-rts}, we first analyse the probability of initialising a population such that there are individuals on each branch with a safety gap of~$\sigma$ to the border between branches. This safety gap will be used to exclude the possibility that the best individual on one branch creates offspring on the opposite branch.

\begin{lemma}
\label{lem:restriselnich-good-init}
Consider the population of the \muea on \twomax and for some $\mu$ and $\sigma$. The probability of having at least one initial search point with at most $n/2 - \sigma$ ones and one search point with at least $n/2 + \sigma$ ones is at least
\[
1-2\left(\frac{1+2\sigma\cdot\sqrt{2/n}}{2}\right)^\mu \ge 1-2^{-\mu+1}\left(1+\Smallo{1}\right)
\]
where the inequality holds if $\sigma \mu = \Smallo{\sqrt{n}}$.
\end{lemma}
\begin{proof}
By \citet[Equation 1.4.17]{Doerr2020}, the following inequalities hold for any binomial coefficient:
\[
\binom{n}{k} \le \binom{n}{\lfloor n/2 \rfloor} \le 2^n \cdot \sqrt{2/n}.
\]
Hence, for a random variable with binomial distribution $\Bin{n,1/2}$, for all $z \in [0,n]$ we have
\begin{align*}
\Prob{X=z}\leq\Prob{X=\lfloor n/2 \rfloor} \le\;& 2^{-n}\cdot\binom{n}{\lfloor n/2 \rfloor} 
\le \sqrt{2/n}.
\end{align*}

So the probability that an individual $x$ is initialised inside the safety gap is at most
\[
p_{\sigma}:=\Prob{n/2-\sigma < \ones{x} < n/2+\sigma}\leq 2\sigma\cdot\sqrt{2/n}.
\]
Now let us define the probability that an individual $x$ is initialised on the outer regions with $\ones{x}\leq n/2 - \sigma$ ones ($0^n$ branch) or $\ones{x}\leq n/2 + \sigma$ ones ($1^n$ branch) as $p_{0}$ and $p_{1}$, respectively. Note that both $p_{0}$ and $p_{1}$ are symmetric, and $p_{0} + p_{1}:=1-p_{\sigma}$, and by rewriting we obtain $p_{0}:=\frac{1-p_{\sigma}}{2}$ (the same for $p_{1}$) with its complement being $1-\frac{1-p_{\sigma}}{2}= \frac{1+p_{\sigma}}{2}$.

So the probability of having no individual with at most $n/2-\sigma$ ones is $(1-p_1)^\mu = \left(\frac{1+p_\sigma}{2}\right)^\mu$, and the same holds for having no individual with at least $n/2 + \sigma$ ones. Hence the probability of being initialised as stated in the statement of the lemma is at least
\[
1-2\left(\frac{1+p_{\sigma}}{2}\right)^\mu = 1 - 2^{-\mu+1} \cdot (1+p_\sigma)^\mu.
\]
Plugging in $p_\sigma$ and using the inequality $1+x\leq e^x$ as well as $\sigma\mu = \Smallo{\sqrt{n}}$ we simplify the last term as
\begin{align*}
 (1+p_\sigma)^\mu \le e^{2\sigma\mu \sqrt{2/n}} 
 =\frac{1}{e^{-\Smallo{1}}}\leq\;& \frac{1}{1-\Smallo{1}} 
 = 1+\Smallo{1},
\end{align*}
and by plugging all together we have $1-2^{-\mu+1}(1+\Smallo{1})$.
\end{proof}

Using Lemmas~\ref{lem:time-mu-n-log-n} and \ref{lem:restriselnich-good-init}, we can now prove Theorem~\ref{the:positive-result-for-rts}.
\begin{proof}[Proof of Theorem~\ref{the:positive-result-for-rts}]
We first apply Lemma~\ref{lem:restriselnich-good-init} with $\sigma := \log n$, noting that the assumption $\sigma \mu = \Smallo{\sqrt{n}}$ holds true since we assume $\mu = \Smallo{\sqrt{n}/\log n}$. According to Lemma~\ref{lem:restriselnich-good-init}, with probability $1-2^{-\mu+1}(1+\Smallo{1})$ the initial population contains at least one search point with at most $n/2-\log n$ ones and at least one search point with at least $n/2+\log n$ ones. We assume in the following that this has happened. Using $k! \ge (k/e)^k = k^{\Omega(k)}$, the probability of mutation flipping at least $\log n$ bits is at most $1/(\log n)! = (\log n)^{-\Bigom{\log n}} = n^{-\Bigom{\log \log n}}$. Taking the union bound over $\Bigo{\mu n \log n}$ steps still gives a superpolynomially small error probability. In the following, we work under the assumption that mutation never flips more than $\log n$ bits.

We call two search points \emph{close} if their genotypic distance is at most $\log n$. Owing to our assumption on mutations, every newly created offspring is close to its parent. Note that on \twomax the phenotypic distance of any two search points is bounded from above by the genotypic distance, hence close search points also have a phenotypic distance of at most $\log n$. Note that, whenever the tournament contains a search point that is close to the new offspring, either the offspring or a close search point will be removed. If this always happens, the best individual on any branch cannot be eliminated by an offspring on the opposite branch; recall that initially, the best search points on the two branches have phenotypic distance at least $2\log n$, and this phenotypic distance increases if the best fitness on any branch improves. When genotypic distances are being used, the genotypic distance is always at least $2\log n$.

Since each offspring has at least one close search point (its parent), the probability that the tournament does not contain any close search point is at most $(1-1/\mu)^w \le e^{-w/\mu}$. Using $w \ge 2.5\mu \ln n$, this is at most $e^{-2.5\ln n} = 1/n^{2.5}$. So long as the best individual on any branch does not get replaced by any individuals on the opposite branch, the conditions of Lemma~\ref{lem:time-mu-n-log-n} are met. In particular, the assumption made in Lemma~\ref{lem:time-mu-n-log-n}, that offspring improving the current best fitness of any branch are always accepted, holds true. Applying Lemma~\ref{lem:time-mu-n-log-n} to both branches, by the union bound the probability of both optima being found in time $2e\mu n \ln n$ is at least $1-2/n$. The probability that in this time a tournament occurs that does not involve a close search point is $\Bigo{\mu n \log n} \cdot 1/n^{2.5} = \Smallo{1/n}$ as $\mu = \Smallo{\sqrt{n}/\log n}$.

All failure probabilities sum up to (assuming $n$ large enough)
\[
\frac{2}{n} + \Smallo{\frac{1}{n}} + 2^{-\mu+1} (1+\Smallo{1}) + \frac{\Bigo{\mu n \log n}}{n^{-\Bigo{\log \log n}}} \le \frac{4}{n} + 2^{-\mu+2}\le 2^{-\mu'+3}
\]
where the last inequality follows as $2^{-\mu} \le 2^{-\mu'}$ and ${1/n \le 2^{-\mu'}}$.
\end{proof}

In Theorem~\ref{the:positive-result-for-rts} we chose $w$ so large that every tournament included the offspring's parent with high probability. Then the \muea behaves like the \muea with deterministic crowding~\citep{Friedrich2009}, leading to similar success probabilities (see Table~\ref{tab:divmech}). 

A success probability around $1-2^{-\mu+1}$ is best possible for many diversity mechanisms as with probability $2^{-\mu+1}$ the whole population is initialised on one branch only (for odd~$n$), and then it is likely that only one optimum is reached. Methods like fitness sharing and clearing obtain success probabilities of~1 by more aggressive methods that can force individuals to travel from one branch to the other by accepting worse search points along the way. The performance of restricted tournament selection (and that of deterministic crowding) is hence best possible amongst all mechanisms that do not allow worse search points to enter the population.

The condition $w \ge 2.5\mu \ln n$ is chosen to ensure that the conditions of Lemma~\ref{lem:time-mu-n-log-n} are met with high probability in one generation, and throughout $O(\mu n \log n)$ generations. This condition is quite crude as it also means that in any fixed generation, the tournament will contain all individuals from the population with high probability. We believe that the algorithm also maintains niches on the two different branches with a much smaller choice of~$w$ as it is sufficient to have any close individual from the offspring's branch in the population to make sure that the offspring only competes against individuals from the same branch. We will investigate this experimentally in Section~\ref{sec:exp_rts}.

\subsection{Small Window Sizes Can Fail}
\label{sec:small-w}

We now turn our attention to small $w$. If the $w$ is small in comparison to $\mu$, the possibility emerges that the tournament only contains individuals that are far from the offspring. In that case even the closest individual in the tournament will be dissimilar to the offspring, resulting in a competition between individuals from different ``niches'' (\ie sets of similar individuals). The following theorem and its proof show that this may result in one branch taking over the other branch, even when the branch to get extinct is very close to a global optimum. The resulting expected runtime is exponential.

\begin{theorem}
\label{the:badresrts}
The probability that the \muea with restricted tournament selection with ${w\ge 3}$ and either genotypic or phenotypic distances finds both optima on \twomax in time $n^{n-1}$ is at most $1 - \exp\left(-\frac{2.28\mu^{w-1}}{n-1}\right) + \Bigo{1/n}
= \Bigo{\mu^{w-1}/n}$. If $\mu \le n^{1/(w-1)}$ then the expected time for finding both optima is $\Bigom{n^n}$.

For $w=2$ the probability is at most $1 - \exp\left(-\frac{2\mu \harm{\mu}}{n}\right) + \Bigo{1/n} = \Bigo{(\mu \log \mu)/n}$, where $\harm{\mu}$ refers to the $\mu$-th harmonic number.
\end{theorem}
Theorem~\ref{the:badresrts} shows an improvement over the preliminary version of this paper~\citep{Covantes2018a}, which gave a weaker  probability bound of $\Bigo{\mu^w/n}$.

Note that the probability of finding both optima in $n^{n-1}$ generations is $\Smallo{1}$ if, for instance, $w = \Bigo{1}$ and $\mu$ grows slower than the polynomial $n^{1/(w-1)}$. It also holds if $w \le c(\ln n)/\ln \ln n$ for some constant $0 < c < 1$ and $\mu = \Bigo{\log n}$ as then $n^{1/(w-1)} = e^{(\ln n)/(w-1)} \ge e^{(\ln \ln n)/c} = (\ln n)^{1/c} = \Smallom{\log n}$, which shows $\mu^{w-1}/n = \Smallo{1}$.
The probability bound of $\Bigo{\mu^{w-1}/n}$ becomes trivial if $w \ge \log_{\mu}(n) + 1 = \log(n)/\log(w) + 1$ as then $\mu^{w-1}/n \ge 1$.

\begin{proof}[Proof of Theorem~\ref{the:badresrts}]
We assume $\mu \le n$ as otherwise all claimed probability bounds are larger than 1 for large enough~$n$.
The analysis follows the proof of Theorem~1 in~\cite{Friedrich2009}. We assume that the initial population contains at most one global optimum as the probability of both optima being found during initialisation is at most $\mu\cdot 2^{-n} \le n \cdot 2^{-n}$, which can be easily subsumed in the terms of $\Bigo{1/n}$ in the claimed probability bounds. 

We consider the first point of time at which the first optimum is being found. Without loss of generality, let us assume that this is $0^n$. Then we show that with high probability copies of $0^n$ take over the population before the other optimum~$1^n$ is found. The following arguments work for all populations that contain $0^n$ but not $1^n$; this includes the most promising population where all $\mu-1$ remaining individuals are Hamming neighbours of~$1^n$. 

Let~$i$ be the number of copies of the~$0^n$ individuals in the population, then a good event~$G_i$ (good in a sense of leading towards extinction as we are aiming at a negative result) is to increase this number from~$i$ to~$i+1$. For this it is just necessary to create copies of one of the~$i$ individuals. For~$n\geq 2$ we have~$\Prob{G_i}\geq\frac{i}{\mu}\cdot\left(1-\frac{1}{n}\right)^n\cdot\left(\frac{\mu-i}{\mu}\right)^w$ since it suffices to select one out of~$i$ individuals and to create a copy of the selected individual, and to select~$w$ times individuals from the remaining~$\mu-i$ individuals. On the other hand, a bad event~$B_i$ is to create an~$1^n$ individual in one generation. This probability is clearly bounded by 
\[
\Prob{B_i} \le \frac{\mu-i}{\mu} \cdot \frac{1}{n} \left(1-\frac{1}{n}\right)^{n-1} + n^{-n}
\]
as the probability to mutate $0^n$ into $1^n$ is $n^{-n}$ and for the remaining $\mu-i$ individuals the chance of creating $1^n$ is at most $1/n \cdot \left(1 - 1/n\right)^{n-1}$.
We deal with the term $n^{-n}$ separately: the probability of such a jump happening in $n^{n-1}$ steps is still at most $1/n$, which can be subsumed in the $\Bigo{1/n}$ term from the claimed failure probability. Hence we ignore this term in the following.
Note that the quotient of both probability bounds is
\begin{align*}
\frac{\Prob{B_i}}{\Prob{G_i}} \le \frac{\frac{\mu-i}{\mu} \cdot \frac{1}{n} \left(1-\frac{1}{n}\right)^{n-1}}{\frac{i}{\mu}\cdot\left(1-\frac{1}{n}\right)^n\cdot\left(\frac{\mu-i}{\mu}\right)^w}
= \frac{1}{n-1} \cdot \frac{\mu^w}{i(\mu-i)^{w-1}}.
\end{align*}
Together, the probability that the good event~$G_i$ happens before the bad event~$B_i$ is
\[
\Prob{G_i\mid G_i\cup B_i}= \frac{\Prob{G_i}}{\Prob{G_i \cup B_i}} \ge 1 - \frac{\Prob{B_i}}{\Prob{G_i} + \Prob{B_i}} \ge \exp(-\Prob{B_i}/\Prob{G_i})
\]
where the last step follows from the well-known inequality $1+\frac{x}{1-x} \ge e^x$ for all $x < 1$ applied to $x := -\Prob{B_i}/\Prob{G_i}$.

The probability that the copies of $0^n$ take over the population before $1^n$ is found is therefore at least
\begin{align}
& \prod_{i=1}^{\mu-1}\Prob{G_i\mid G_i\cup B_i}\geq
\prod_{i=1}^{\mu-1}\exp\left(-\frac{\Prob{B_i}}{\Prob{G_i}}\right)\notag\\
=\;& \exp\left(-\sum_{i=1}^{\mu-1}\frac{\Prob{B_i}}{\Prob{G_i}}\right)
\ge \exp\left(-\frac{\mu^{w}}{n-1}\sum_{i=1}^{\mu-1}\frac{1}{i\left(\mu-i\right)^{w-1}}\right).\label{eq:sum}
\end{align}
For $w=2$ the last sum simplifies to
\[
\sum_{i=1}^{\mu-1} \frac{1}{i(\mu-i)} 
= \sum_{i=1}^{\mu-1} \left(\frac{1}{i\mu} + \frac{1}{\mu(\mu-i)}\right) 
=\frac{2\harm{\mu-1}}{\mu}
\]
and the probability that takeover happens is at least
\[
\exp\left(-\frac{\mu^w}{n-1} \cdot \frac{2\harm{\mu-1}}{\mu}\right) \ge \exp\left(-\frac{2\mu \harm{\mu}}{n}\right)
\]
where the inequality holds since $w=2$ and $(\mu-1)/(n-1) \le \mu/n$, which in turn is implied by $\mu \le n$. Along with the error term of $\Bigo{1/n}$, this yields the claimed probability bound for $w=2$.

For $w\ge 3$ we note that the summands in~\eqref{eq:sum} are non-increasing with $w$. So the worst case is having the smallest possible value, $w = 3$. Note that
\[
\frac{1}{i(\mu-i)^2} 
= \frac{\mu-i}{i\mu(\mu-i)^2} + \frac{i}{i\mu(\mu-i)^2}
= \frac{1}{i\mu(\mu-i)} + \frac{1}{\mu(\mu-i)^2}
= \frac{1}{\mu} \left(\frac{1}{i(\mu-i)} + \frac{1}{(\mu-i)^2}\right).
\]
Thus
\begin{align*}
\sum_{i=1}^{\mu-1} \frac{1}{i(\mu-i)^2}
=\;& \frac{1}{\mu} \left(\sum_{i=1}^{\mu-1} \frac{1}{i(\mu-i)} + \sum_{i=1}^{\mu-1} \frac{1}{(\mu-i)^2}\right)\\
=\;& \frac{1}{\mu} \left(\frac{2\harm{\mu-1}}{\mu} + \sum_{i=1}^{\mu-1} \frac{1}{(\mu-i)^2}\right)\\
=\;& \frac{2\harm{\mu-1}}{\mu^2} + \frac{1}{\mu} \sum_{i=1}^{\mu-1} \frac{1}{i^2}.
\end{align*}
Since $\sum_{i=1}^{\mu-1} 1/i^2 \le \sum_{i=1}^\infty 1/i^2 = \pi^2/6 \approx 1.645$, the above approaches $1.645/\mu$ as $\mu$ grows. 
The function is $2/\mu$ for $\mu=2$, $2.25/\mu$ for $\mu=3$, and at most $2.28/\mu$ for $\mu \ge 4$, with the constant factor decreasing with increasing~$\mu$ for $\mu \in [4, \infty)$. Hence the function is bounded by $2.28/\mu$ for all~$\mu \in \mathbb{N}$.

Together we have
\[
\prod_{i=1}^{\mu}\Prob{G_i\mid G_i\cup B_i}\geq\exp\left(-\frac{\mu^{w}}{n-1} \cdot \frac{2.28}{\mu}\right)= \exp\left(-\frac{2.28\mu^{w-1}}{n-1}\right)\ge 1-\Bigo{\mu^{w-1}/n}.
\]

Once the population consists only of copies of $0^n$, a mutation has to flip all $n$ bits to find the $1^n$ optimum. This event has probability $n^{-n}$ and, by the union bound, the probability of this happening in a phase consisting of $n^{n-1}$ generations is at most $\frac{1}{n} = \Bigo{\mu^{w-1}/n}$. The sum of all failure probabilities is $\Bigo{\mu^{w-1}/n}$, which proves the first claim. For the second claim, observe that the conditional expected runtime is $n^n$ once the population has collapsed to copies of $0^n$ individuals. 
Using $\mu \le n^{1/(w-1)}$ this situation  
occurs with probability at least 
$\exp\left(-\frac{2.28n}{n-1}\right) - \Bigo{1/n} = \Bigom{1}$. Hence the unconditional expected runtime is $\Bigom{n^n}$.
\end{proof}

\section{Runtime Guarantees for RTS without Replacement}
\label{sec:rts_nor}

In this section we now discuss what happens if RTS is modified to select $w$ individuals \emph{without replacement}, that is, $w$ \emph{different} individuals are being selected. 
We call this a modification as we believe the original RTS selects with replacement (this is not mentioned explicitly, but it follows from the mathematical formulae in~\cite{Harik1995}). Selecting without replacement makes as much sense as selecting with replacement, and it is very plausible that many practical implementations have used one or the other variant. Selecting without replacement leads to a more diverse tournament and we expect a stronger effect, compared to selecting with replacement, when the same window size is used.
Note that if $w=\mu$ then the whole population is selected for the tournament and thus a closest individual in the population is selected to compete against the offspring. We define the algorithm for $w > \mu$ as well, for consistency with experiments for RTS with replacement (that use values for~$w$ larger than~$\mu$), even though there is no difference to $w=\mu$. If $w>\mu$ we select the whole population and $w-\mu$ copies of arbitrary individuals; the effect is the same as for $\mu=w$.

\subsection{Large Window Sizes Still Work}
\label{sec:large_w_rts_nor}
The positive result mentioned in Section~\ref{sec:large-w} still applies for this variant of RTS. 
The main difference to Theorem~\ref{the:positive-result-for-rts} is that, assuming the algorithm never flips at least $\log n$ bits, if $w \ge \mu$ the \muea always selects at least one close search point for the tournament (as opposed to a probability of at least $1-1/n^{2.5}$). We thus obtain the same result as  Theorem~\ref{the:positive-result-for-rts} for the weaker condition $w \ge\mu$. Note that this condition implies that the whole population is contained in the tournament, and so the offspring always competes against the closest individual from the population.

\begin{theorem}
\label{the:positive-result-for-rts-without-replacement}
If $\mu = \Smallo{\sqrt{n}/\log n}$ and $w \ge \mu$ the \muea with restricted tournament selection selecting $w$ individuals without replacement, using genotypic or phenotypic distance, finds both optima on \twomax in time $\Bigo{\mu n \log n}$ with probability at least ${1-2^{-\mu'+3}}$, where $\mu' := \min(\mu, \log n)$.
\end{theorem}

\subsection{On Takeover with Small Window Sizes}
\label{sec:small_w_rts_nor}

Unlike our positive result that works for both selection policies, the negative result (Theorem~\ref{the:badresrts}) is no longer applicable for RTS without replacement. Recall that for the negative result we rely on an extreme case where the algorithm has found the first optimum $0^n$ and then this optimum takes over the whole population before the opposite optimum $1^n$ is found. 
Consider the situation where there is just one individual~$x$ left (a Hamming neighbour of $1^n$) before copies of $0^n$ take over and the new offspring~$y$ is another copy of $0^n$. Theorem~\ref{the:badresrts} then relied on individual~$x$ being selected~$w$ times for the tournament in order to complete the takeover as then $x$ is the closest individual to~$y$. 
When using RTS without replacement, this event is impossible for $w \ge 2$ as the tournament is guaranteed to contain a copy of $0^n$, which then competes against the new offspring $y$. 
Hence, in this case, the subpopulation on the branch towards $1^n$ will never become extinct and will eventually reach $1^n$.

In fact, the following lemma shows that if the \muea using RTS without replacement and phenotypic distances is able to maintain subpopulations on both branches until both subpopulations have evolved to a fitness larger than $2n/3$, extinction becomes impossible and it is certain that both optima will be found eventually. For genotypic distances a similar statement holds when all individuals have passed a (higher) fitness threshold of $3n/4$.
\begin{lemma}
\label{lem:rts-without-replacement-extinction-impossible}
Consider the \muea using restricted tournament selection without replacement and $w \ge 2$. If phenotypic distances are used, once a population is reached that has subpopulations on both branches
whose best fitness is larger than $2n/3$, no subpopulation will become extinct.
If genotypic distances are used, once a population is reached that has subpopulations on both branches
whose worst fitness is larger than $3n/4$, no subpopulation will become extinct.
\end{lemma}
\begin{proof}
A subpopulation can only become extinct if it only contains a single individual. We hereinafter call this individual~$\lone$ and observe that by assumption $f(\lone) > 2n/3$. Without loss of generality, we assume that~$\lone$ is on the 1-branch, thus $\ones{\lone} > 2n/3$.
The following events are necessary for extinction: an offspring~$y$ of fitness $f(y) \ge f(\lone)$ is created on the opposite branch to~$\lone$ (i.\,e., $\ones{y} < n/3$), the single individual~$\lone$ is chosen for the tournament, and~$\lone$ is the closest search point to the offspring~$y$. Only then will~$\lone$ be removed from the population. 

We show that these conditions are impossible. The phenotypic distance between~$\lone$ and the offspring~$y$ is larger than $n/3$. Since $w \ge 2$, the tournament must also contain an individual~$x'$ from the 0-branch. If $f(x') < f(y)$ then $x'$ is phenotypically closer to~$y$ than~$\lone$. Otherwise, since $0 \le \ones{x'} \le \ones{y} < n/3$, the phenotypic distance between $x'$ and the offspring~$y$ is less than $n/3$, hence~$\lone$ cannot be the closest search point to the offspring~$y$. Consequently, $\lone$ cannot be replaced by an offspring on the opposite branch.

With genotypic distances the same arguments apply. Assume that $\ones{\lone} > 3n/4$ and $f(y) \ge f(\lone)$, i.\,e., $\ones{y} < n/4$. Then the genotypic distance between every individual $x'$ from the 0-branch and the offspring~$y$ is less than $n/2$ as by assumption $\ones{x'} < n/4$ and thus both $y$ and $x'$ have Hamming distance less than $n/4$ to $0^n$, whereas the genotypic distance between $\lone$ and the offspring~$y$ is larger than $n/2$.
\end{proof}

In scenarios where the fitness thresholds from Lemma~\ref{lem:rts-without-replacement-extinction-impossible} are not reached, takeover of one branch may still happen. Imagine a population where an offspring~$y$ is created on one branch and the tournament contains an individual~$z$ from the opposite branch with $f(z) < f(y)$. If the tournament only contains other search points whose distance to~$y$ is larger than the distance between $z$ and $y$, $z$ will be removed from the population. If such steps happen repeatedly, the subpopulation on $z$'s branch may become extinct. Note, however, that if the size of said subpopulation is less than~$w$, the tournament must contain individuals from $y$'s branch that have a large distance from~$y$.

Figure~\ref{fig:twomax_takeover} sketches a population described above. The population is divided into two species. On the $1$-branch there is only one individual (orange point in Figure~\ref{fig:twomax_takeover} or individual $z$) and on the $0$-branch there are $\mu-1$ individuals and the offspring (red point in Figure~\ref{fig:twomax_takeover} or individual $y$), which has a better fitness than $z$. The parent individuals on the 0-branch contain a subset $P'$ of individuals whose distance to~$y$ is larger than the distance between $z$ and $y$. 
In this case, if the tournament consists of $z$ and $w-1$ individuals from $P'$ then $z$ is the closest individual to $y$ and $y$ replaces $z$.

\begin{figure*}[!ht]
	\centering
    \resizebox{.70\linewidth}{!}{
		\begin{tikzpicture}[domain=0:30,xscale=0.15,yscale=0.15,scale=1, every shadow/.style={shadow
        xshift=0.0mm, shadow yshift=0.4mm}]
          \tikzstyle{species1}=[draw=blue, fill=blue!30!white, line width=0.2ex];
          \tikzstyle{helpline}=[black,thick];
          \tikzstyle{function}=[blue,very thick];
          \tikzstyle{individual}=[green,very thick];
          \tikzstyle{loneind}=[orange,very thick];
          \tikzstyle{winp}=[brown,very thick];
          \tikzstyle{offspring}=[red,very thick];
          \draw[gray!20,line width=0.2pt,xstep=1,ystep=1] (0,0) grid (30.5,15.5);
          \draw[function] (0,15) -- (15,0) -- (30,15);
          \draw[helpline, thin, -triangle 45] (0,0) -- (0,17);
          \draw[helpline, thin, -triangle 45] (0,0) -- (32,0) node[right] {\scriptsize \#{}ones};
          \draw[helpline] (0,0) -- (0,16);
          \draw[helpline] (0,0) -- (31,0);
          \draw[helpline] (0,0) -- (-1,0) node[left] {\scriptsize $n/2$};
          \draw[helpline] (0,0) -- (0,-1) node[below] {\scriptsize $0$};
          \draw[helpline] (0,15) -- (-1,15) node[left] {\scriptsize $n$};
          \draw[helpline] (15,0) -- (15,-1) node[below] {\scriptsize $n/2$};
          \draw[helpline] (30,0) -- (30,-1) node[below] {\scriptsize $n$};
          
          \draw[gray, densely dotted, thick, arrows={|[width=3pt]-|[width=3pt]}] (12, 2) -- (17, 2); 

          \node (z) at (18.5, 2) {\scriptsize $z$};
          \filldraw[loneind] (17,2) circle (6pt); 
          
          \node (y) at (10.5, 3) {\scriptsize $y$};
          \filldraw[offspring] (12,3) circle (6pt); 
          \filldraw[individual] (13,2) circle (6pt);
          \filldraw[individual] (14,1) circle (6pt);
          
          \node (p) at (6.5, 12.5) {\scriptsize $P'$};
          \draw[species1,rotate around={135:(3.5, 11.5)}] (3.5, 11.5) ellipse (3cm and 1cm);
          \draw[gray, densely dotted, thick, arrows={|[width=3pt]-|[width=3pt]}] (5, 10) -- (12, 10); 

          \filldraw[individual] (2,13) circle (6pt);
          \filldraw[individual] (3,12) circle (6pt);
          \filldraw[individual] (4,11) circle (6pt);
          \filldraw[individual] (5,10) circle (6pt);
          
		\end{tikzpicture}
    }
	\caption{Sketch of $f=\twomax$ with a population where takeover may happen for the \muea with RTS selecting without replacement, using phenotypic distances. The offspring ($y$, red) replaces the single individual ($z$, orange) if the tournament contains $z$ and $w-1$ individuals from $P'$ as then $z$ is the one closest individual from the tournament as indicated by the dotted lines.}
	\label{fig:twomax_takeover}
\end{figure*}
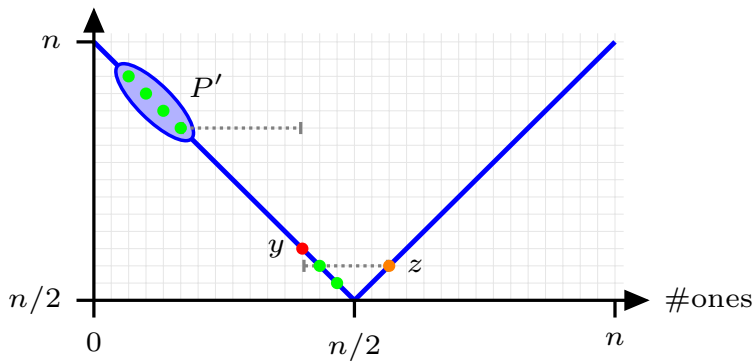


Note that the scenario exemplified in Figure~\ref{fig:twomax_takeover} is impossible under the conditions from Lemma~\ref{lem:rts-without-replacement-extinction-impossible} as there $z$ would have a much higher fitness and the niche $P'$ with the mentioned properties cannot exist. It is also highly unlikely for such a population to emerge in the \muea without any diversity mechanisms as the subpopulations on both branches tend to have similar fitness values. When using RTS, however, it is possible for different niches like $P'$, and niches with significantly worse fitness, to emerge.
We will further investigate the likelihood of takeover in the experimental analysis (Section~\ref{sec:exp_rts}).

\subsection{RTS with Small Window Sizes Slows Down Evolution}
\label{sec:rts-slowdown}

As mentioned in previous sections, RTS without replacement is proven to be effective for smaller values of $w$. But this comes at a price as the time to reach both optima can increase significantly. In a scenario where there is just one individual~$\lone$ on one branch, this ``lone'' individual will need to climb up its branch until it reaches its respective optimum. 
Assume we have a generation where~$\lone$ was selected as a parent and has produced a better offspring on the same branch. Further assume that~$\lone$ and its offspring have a worse fitness than all individuals on the opposite branch. Then the only way that the subpopulation of~$\lone$ can evolve is if $\lone$ is selected in the tournament. In this case, $\lone$ will compete against its (fitter) offspring and will be removed from the population. In other words, $\lone$ has to be removed to allow its offspring to survive. This means that~$\lone$ will need to be selected twice in one generation: as parent and in the replacement selection. Also note that the new subpopulation will consist of just one individual, hence the algorithm may be in the same situation for a long period of time. 


We make this precise in the following lemma which shows that the evolution on a branch with only a single individual proceeds as in a \emph{lazy}\footnote{In Markov chain theory, a Markov chain is called \emph{lazy} if there is a fixed probability (\eg $1/2$) of remaining in the same state~\citep{LPW09}.} (slowed-down) \eaoneone. In addition to assuming a branch with a single individual~$x$, we assume that all individuals on the opposite branch will have a fitness that is larger than the fitness of~$x$ by an amount that is at least logarithmic.

\begin{definition}
Define the $p$-lazy {\EA} as an algorithm that independently in each iteration idles with probability~$p$ and otherwise (that is, with probability $1-p$) performs one step of the \EA.
\end{definition}

The following lemma assumes without loss of generality that the lone individual resides on the 1-branch. A symmetric statement holds when the roles of the two branches are swapped.
\begin{lemma}
\label{lem:lazy-oneoneea}
Consider the \muea using restricted tournament selection without replacement, window size~$w \le \mu$ and genotypic or phenotypic distance. Suppose that the population $P_t$ at time~$t$ contains a single individual $x$ on the 1-branch with fitness at least $n/2 + \log n$ and that all other individuals in $P_t$ are on the 0-branch and have a fitness of at least $f(x) + \log n$.
Then, with probability $1-n^{-\Smallom{1}}$, the subpopulation on the 1-branch in the \muea will evolve as in one step of the $(w/\mu^2)$-lazy \EA with current search point~$x$ on \onemax.
\end{lemma}
\begin{proof}
If any of the search points on the 0-branch is selected as parent, at least $\log n$ bits have to flip to create a search point on the 1-branch of fitness at least $f(x)$. This has probability $n^{-\Smallom{1}}$ and it is a necessary condition for the subpopulation on the 1\nobreakdash-branch to change (as a worse individual on the 1-branch will be removed regardless of the outcome of the tournament).

If $x$ is selected as parent, at least $\log n$ bits have to flip to create a search point that is on the 0-branch or at least as good as the worst individual on the 0-branch. This, again, has probability $n^{-\Smallom{1}}$. If this does not happen, the offspring $y$ can only survive if the tournament contains~$x$ and $f(y) \ge f(x)$. In this case, $y$ replaces~$x$. The probability for the tournament containing~$x$ is $w/\mu$ (there are $\binom{\mu}{w}$ ways of choosing $w$ individuals without replacement and $\binom{\mu-1}{w-1}$ ways of choosing $x$ and $w-1$ other individuals from the remaining $\mu-1$ individuals; together, this yields a probability of $\binom{\mu-1}{w-1}/\binom{\mu}{w} = w/\mu$.). Along with a probability of $1/\mu$ for selecting~$x$ as parent, the probability of $y$ replacing $x$ in case $f(y) \ge f(x)$ is $w/\mu^2$.
\end{proof}

Lemma~\ref{lem:lazy-oneoneea} (along with the well-known fact that the \EA requires $\Bigth{n \log n}$ time on \onemax) suggests that, under appropriate conditions, a lower bound of $\Bigom{(\mu^2/w) \cdot  n \log n}$ applies for the \muea with RTS with replacement on \twomax. However, to formally prove such a bound, we would need to show that the assumptions of the lemma have a good chance to be satisfied during an appropriate time period. 

In preliminary experiments we observed that, for $w=2$ and across a range of values for~$\mu$, a lone individual emerged in almost all runs. Furthermore, in those cases the individuals on the opposite branch evolved faster, and the time to find both optima was determined by the time the lone individual evolved its respective optimum. 
We do not have a formal proof that the larger subpopulation evolves faster, though. And it is not clear whether a lone individual typically develops for larger values of~$w$. We therefore resort to experiments in Section~\ref{sec:exp_rts} to investigate this matter further and to check for which values of~$w$ our conjectured lower bound of $\Bigom{(\mu^2/w) \cdot  n \log n}$ is supported.



\section{Experimental Analysis}
\label{sec:exp_rts}


We provide an experimental analysis as well in order to see how closely the theory matches the empirical performance for a reasonable problem size, and to investigate a wider range of parameters, where the theoretical results are not applicable. 
We are interested in the impact of the window size~$w$ and the selection policy (with or without replacement) on the success rate of the \muea with RTS.


Another interesting question is to compare RTS with and without replacement in its resilience to takeover. We argued in Section~\ref{sec:rts_nor} that the RTS without replacement is more resilient to takeover, compared to its variant with replacement, but at the expense of an increased runtime for finding both optima. We also argued that there are scenarios where takeover may happen when no replacement is used, so we would like to know for how long RTS is able to delay takeover and when it is more likely that takeover happens for both selection policies.

We consider exponentially increasing population sizes ${\mu\in\{2,4,8,\ldots,1024\}}$ for a problem size $n=100$ and for $100$ runs. Based on our theoretical analysis we define the following outcomes and stopping criterion for each run. \emph{Success}, the population contains both $0^n$ and $1^n$ in the population. And \emph{failure}, once the run has reached a maximum number of generations and the population does not contain both optima.
This maximum is initially set to $10\mu n\ln n$ as motivated by Lemma~\ref{lem:time-mu-n-log-n}, with a more generous leading constant that leaves plenty of time for most diversity mechanisms to find both optima in case no takeover happens. The same time bound was also used in~\cite{Covantes2018b}, an empirical comparison of a range of diversity mechanisms. In Section~\ref{sec:extra_time_budget} we will consider a larger time budget to reflect our conjecture that RTS without replacement requires more time by a factor of order $\mu^2/w$, for appropriate values of~$w$.

For both RTS variants, 
we tested window sizes $w\in\{1,2,4,8,\ldots,1024\}$, however we only plot results up to~$w=128$ as the results for large~$w$ were very similar. In the particular case where $w>\mu$ and the tournament is selected without replacement, the algorithm uses the whole population in the tournament.

\subsection{\boldmath Experiments for a Time Budget of $10\mu n\ln n$}
\label{sec:well-known-time-budget}

In Figures~\ref{fig:successrtsg} and \ref{fig:successrtsp} we show the results for the \muea with the original RTS, selecting~$w$ individuals with replacement. For small values of~$w$ and $\mu$ the algorithm is not able to maintain individuals on both branches of \twomax for a long period of time, as predicted by Theorem~\ref{the:badresrts}. It is only when the population size is set to $\mu=1024$ and $w=1$ that the algorithm is able to maintain individuals on both branches before the takeover happens. When setting, for example, $w\geq 8$ and $\mu\geq 32$, the algorithm was able to find both optima with both genotypic and phenotypic distances. It is possible to observe a trade-off between $w$ and~$\mu$: a larger~$w$ allows for a smaller population size~$\mu$ to be used. Such a trade-off was also indicated by the probability bound $\Bigo{\mu^{w-1}/n}$ from Theorem~\ref{the:badresrts}. 

Our experiments show that RTS works well for much smaller window sizes than those required in Theorem~\ref{the:positive-result-for-rts}. For instance, for $\mu=32$ Theorem~\ref{the:positive-result-for-rts} results a window size of $w \ge 369$, whereas Figure~\ref{fig:successrtspnr} shows that $w \ge 8$ seems sufficient to be effective.
Note that the method seems to behave fairly similarly with respect to both distance functions.

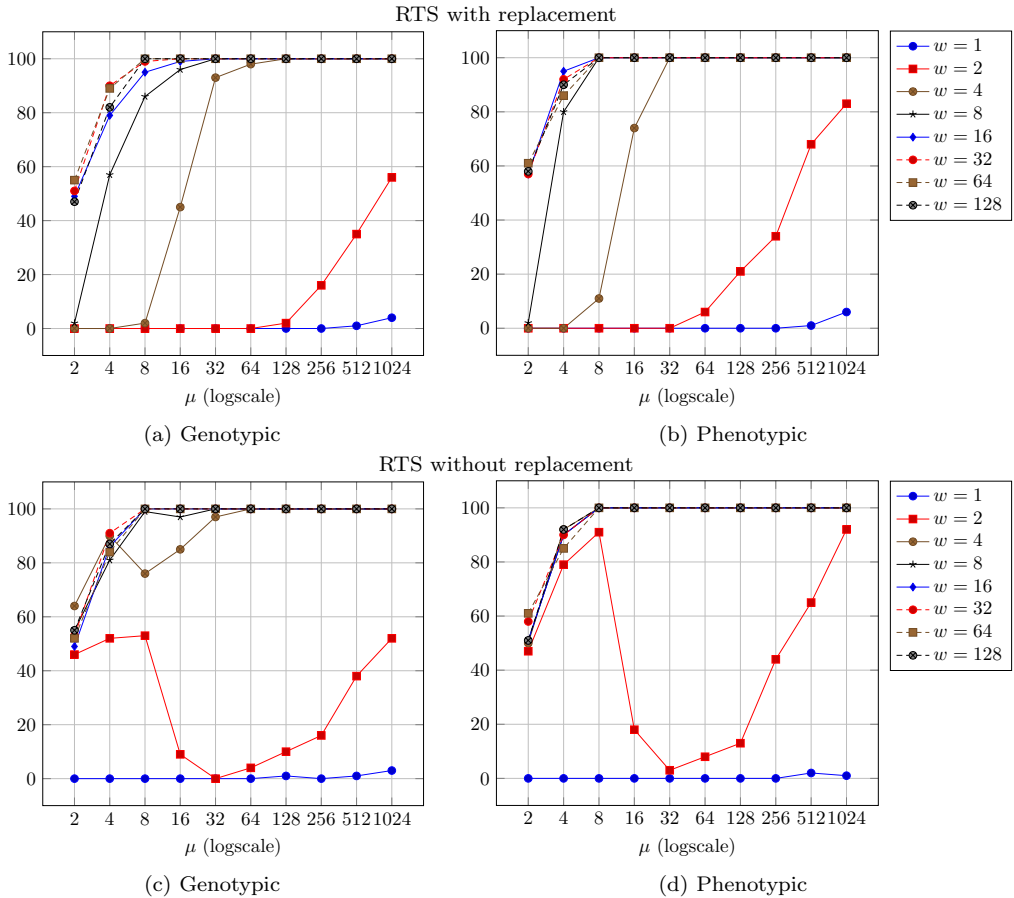
\begin{figure}[!ht]
    \centering
    \begin{tabular}{@{}c@{}c@{}}
        \multicolumn{2}{c}{\footnotesize RTS with replacement} \\
        \subfloat[Genotypic]{
            \resizebox{.382\linewidth}{!}{
                \begin{tikzpicture}	
        		    \begin{axis}[
                        width=7cm,
            		    height=6cm,
                        xlabel=$\mu$ (logscale),
            		    xtick={1,2,3,4,5,6,7,8,9,10},
                        xticklabels={2,4,8,16,32,64,128,256,512,1024},
                        ymin=0, ymax=100,
            		    every axis y label/.style={rotate=0, black, at={(-0.13,0.5)},},
    				]
    					\addplot table[x=mu, y=w1]{\successrtsg};
                    	\addplot table[x=mu, y=w2]{\successrtsg};
                        \addplot table[x=mu, y=w4]{\successrtsg};
                        \addplot table[x=mu, y=w8]{\successrtsg};
                        \addplot table[x=mu, y=w16]{\successrtsg};
                        \addplot table[x=mu, y=w32]{\successrtsg};
                        \addplot table[x=mu, y=w64]{\successrtsg};
                        \addplot table[x=mu, y=w128]{\successrtsg};
        		   \end{axis}
    		    \end{tikzpicture} 
		    }
		    \label{fig:successrtsg}
		} 
        & 
        \subfloat[Phenotypic]{
            \resizebox{.5\linewidth}{!}{
                \begin{tikzpicture}	
            	    \begin{axis}[
            	        width=7cm,
            	        height=6cm,
        		        xlabel=$\mu$ (logscale),
        		        xtick={1,2,3,4,5,6,7,8,9,10},
                        xticklabels={2,4,8,16,32,64,128,256,512,1024},
                        ymin=0, ymax=100,
        		        every axis y label/.style={rotate=0, black, at={(-0.13,0.5)},},
        		        legend columns=1,
                        legend style={at={(1.35,1)},anchor=north east},
                    	legend cell align=left,
                	]
        		        \addplot table[x=mu, y=w1]{\successrtsp};
                        \addplot table[x=mu, y=w2]{\successrtsp};
                        \addplot table[x=mu, y=w4]{\successrtsp};
                        \addplot table[x=mu, y=w8]{\successrtsp};
                        \addplot table[x=mu, y=w16]{\successrtsp};
                        \addplot table[x=mu, y=w32]{\successrtsp};
                        \addplot table[x=mu, y=w64]{\successrtsp};
                        \addplot table[x=mu, y=w128]{\successrtsp};
                        \legend{$w=1$, $w=2$, $w=4$, $w=8$, $w=16$, $w=32$, $w=64$, $w=128$, $w=256$, $w=512$, $w=1024$};
        	        \end{axis}
    	        \end{tikzpicture}
	        }
	        \label{fig:successrtsp}
        } \\
        \multicolumn{2}{c}{\footnotesize RTS without replacement} \\
        \subfloat[Genotypic]{
            \resizebox{.382\linewidth}{!}{
                \begin{tikzpicture}	
        		    \begin{axis}[
                        width=7cm,
                    	height=6cm,
                        xlabel=$\mu$ (logscale),
                    	xtick={1,2,3,4,5,6,7,8,9,10},
                        xticklabels={2,4,8,16,32,64,128,256,512,1024},
                    	every axis y label/.style={rotate=0, black, at={(-0.13,0.5)},},
    				]
            			\addplot table[x=mu, y=w1]{\successrtsgnr};
                        \addplot table[x=mu, y=w2]{\successrtsgnr};
                        \addplot table[x=mu, y=w4]{\successrtsgnr};
                        \addplot table[x=mu, y=w8]{\successrtsgnr};
                        \addplot table[x=mu, y=w16]{\successrtsgnr};
                        \addplot table[x=mu, y=w32]{\successrtsgnr};
                        \addplot table[x=mu, y=w64]{\successrtsgnr};
                        \addplot table[x=mu, y=w128]{\successrtsgnr};
        		    \end{axis}
    		    \end{tikzpicture}
		    }
		    \label{fig:successrtsgnr}
        } 
        &
        \subfloat[Phenotypic]{
            \resizebox{.5\linewidth}{!}{
                \begin{tikzpicture}	
            	    \begin{axis}[
            		    width=7cm,
            		    height=6cm,
            		    xlabel=$\mu$ (logscale),
            		    xtick={1,2,3,4,5,6,7,8,9,10},
                        xticklabels={2,4,8,16,32,64,128,256,512,1024},
                    	every axis y label/.style={rotate=0, black, at={(-0.13,0.5)},},
                        legend columns=1,
                        legend style={at={(1.35,1)},anchor=north east},
                        legend cell align=left,
                    ]
                    	\addplot table[x=mu, y=w1]{\successrtspnr};
                        \addplot table[x=mu, y=w2]{\successrtspnr};
                        \addplot table[x=mu, y=w4]{\successrtspnr};
                        \addplot table[x=mu, y=w8]{\successrtspnr};
                        \addplot table[x=mu, y=w16]{\successrtspnr};
                        \addplot table[x=mu, y=w32]{\successrtspnr};
                        \addplot table[x=mu, y=w64]{\successrtspnr};
                        \addplot table[x=mu, y=w128]{\successrtspnr};
                        \legend{$w=1$, $w=2$, $w=4$, $w=8$, $w=16$, $w=32$, $w=64$, $w=128$, $w=256$, $w=512$, $w=1024$};
            	    \end{axis}
    		    \end{tikzpicture}
		    }
		    \label{fig:successrtspnr}
        }
    \end{tabular}
    \caption{The number of successful runs measured among $100$ runs at the time both optima were found on \twomax or $t=10\mu n \ln n$ generations have been reached for $n=100$ with the \muea with restricted tournament selection with and without replacement for $\mu\in\{2,4,8,\ldots,1024\}$, $w\in\{1,2,4,8,\ldots,128\}$, genotypic and phenotypic distance.}
    \label{fig:successrtsoldrunt}
\end{figure}

Now, Figures~\ref{fig:successrtsgnr} and \ref{fig:successrtspnr} illustrate the performance of RTS without replacement. While all curves for RTS with replacement and $w \ge 8$ in Figure~\ref{fig:successrtsgnr} were monotonically increasing, the success rate for $w=2$ without replacement in Figure~\ref{fig:successrtspnr} is clearly not monotonic. For population sizes $\mu\in\{2,4,8\}$ the success rate does increase, but once the population size increases to $\mu=16$ or ${\mu=32}$, the success rate drops steeply. Finally, the success rate starts going up again when ${\mu\ge 64}$. 

From our theoretical considerations and observing individual runs, the reason for this drop in the success rate seems to be due to the scenarios described in Section~\ref{sec:rts-slowdown}. During the runs, if there is a subpopulation of individuals with better fitness than the individuals on the opposite branch of \twomax, the better subpopulation starts taking over until there is just one individual~$x$ in the worse subpopulation. With $w = 2$ and using selection without replacement, in order for an improving mutation~$y$ of~$x$ to be accepted, the tournament must contain~$x$, assuming all individuals in the better subpopulation have a higher fitness than~$y$. Hence, for the worse subpopulation of one individual to progress, the single individual must be selected as parent and also be selected for the tournament as deleting~$x$ is the only way for its offspring to survive. 
Compared to a setting with large window sizes (Theorems~\ref{the:positive-result-for-rts} and~\ref{the:positive-result-for-rts-without-replacement}), this indicates an additional factor of $\mu/w$ in the expected time for the worst subpopulation to reach its optimum (cf. Lemma~\ref{lem:lazy-oneoneea}).  
Hence the drop in success rates seems to be caused by the time budget not being large enough to allow a single individual on one branch to evolve into its respective optimum (we will investigate this further in the following by increasing the time budget). This behaviour only occurs for intermediate population sizes~$\mu$ as for small~$\mu$ the difference between factors of $\mu$ and $\mu^2/w$ is not significant enough (recall that we have chosen a generous leading constant of 10 in the time bound $10\mu n \ln n$). For large~$\mu$, the algorithm is efficient, albeit less efficient than for larger values of the window size~$w$.

\subsection{\boldmath An Increased Time Budget of $100\max(1, \mu/w)\mu n \ln n$}
\label{sec:extra_time_budget}

Since we hypothesise that RTS without replacement requires more time to achieve positive results (see the theoretical results from Section~\ref{sec:rts-slowdown}, particularly Lemma~\ref{lem:lazy-oneoneea}), we performed the same experiments but increasing the time bound by an additional factor of $\max(1, \mu/w)$, where the maximum is used to accommodate values of $w > \mu$ that show identical behaviour to $w=\mu$. We also increased the leading constant from 10 to 100 to ensure that success or failure does happen in the given time, resulting in an enhanced time budget of $100\max(1, \mu/w)\mu n \ln n$. Our experiments will confirm that this budget
was sufficient in all runs.

In the case of RTS with replacement with the new time bound (Figures~\ref{fig:successrtsgnt} and \ref{fig:successrtspnt}), there are no major changes compared to the previous results shown in Figures~\ref{fig:successrtsg} and \ref{fig:successrtsp}. This suggests that the original time budget is sufficient for the algorithm to arrive at the defined outcomes (success or failure). The major difference is shown with respect to RTS without replacement with the new time budget (Figures~\ref{fig:successrtsgnrnt} and \ref{fig:successrtspnrnt}) compared to the same algorithm with the original time budget (Figures~\ref{fig:successrtsgnr} and \ref{fig:successrtspnr}). As can be observed the performance shown in Figures~\ref{fig:successrtsgnrnt} and \ref{fig:successrtspnrnt} show that RTS without replacement has a better performance for small values $w\ge 2$ and $\mu\ge 16$ when given a larger time budget. 

\begin{figure}[!ht]
    \centering
    \begin{tabular}{@{}c@{}c@{}}
        \multicolumn{2}{c}{\footnotesize RTS with replacement} \\
        \subfloat[Genotypic]{
        \resizebox{.382\linewidth}{!}{
        \begin{tikzpicture}	
    		\begin{axis}[
                width=7cm,
        		height=6cm,
                xlabel=$\mu$ (logscale),
        		xtick={1,2,3,4,5,6,7,8,9,10},
                xticklabels={2,4,8,16,32,64,128,256,512,1024},
        		every axis y label/.style={rotate=0, black, at={(-0.13,0.5)},},
				]
				\addplot table[x=mu, y=w1]{\successrtsgwrnt};
                \addplot table[x=mu, y=w2]{\successrtsgwrnt};
                \addplot table[x=mu, y=w4]{\successrtsgwrnt};
                \addplot table[x=mu, y=w8]{\successrtsgwrnt};
                \addplot table[x=mu, y=w16]{\successrtsgwrnt};
                \addplot table[x=mu, y=w32]{\successrtsgwrnt};
                \addplot table[x=mu, y=w64]{\successrtsgwrnt};
                \addplot table[x=mu, y=w128]{\successrtsgwrnt};
    		\end{axis}
		\end{tikzpicture}
		}
		\label{fig:successrtsgnt}
    } 
    & 
    \subfloat[Phenotypic]{
    \resizebox{.5\linewidth}{!}{
        \begin{tikzpicture}	
        	\begin{axis}[
        		width=7cm,
        		height=6cm,
        		xlabel=$\mu$ (logscale),
        		xtick={1,2,3,4,5,6,7,8,9,10},
                xticklabels={2,4,8,16,32,64,128,256,512,1024},
        		every axis y label/.style={rotate=0, black, at={(-0.13,0.5)},},
                legend columns=1,
                legend style={at={(1.35,1)},anchor=north east},
                legend cell align=left,
        	]
        		\addplot table[x=mu, y=w1]{\successrtspwrnt};
                \addplot table[x=mu, y=w2]{\successrtspwrnt};
                \addplot table[x=mu, y=w4]{\successrtspwrnt};
                \addplot table[x=mu, y=w8]{\successrtspwrnt};
                \addplot table[x=mu, y=w16]{\successrtspwrnt};
                \addplot table[x=mu, y=w32]{\successrtspwrnt};
                \addplot table[x=mu, y=w64]{\successrtspwrnt};
                \addplot table[x=mu, y=w128]{\successrtspwrnt};

                \legend{$w=1$, $w=2$, $w=4$, $w=8$, $w=16$, $w=32$, $w=64$, $w=128$, $w=256$, $w=512$, $w=1024$};
        	\end{axis}
		\end{tikzpicture}  	
		}
        \label{fig:successrtspnt}
    }
    \\
    \multicolumn{2}{c}{\footnotesize RTS without replacement} \\
    \subfloat[Genotypic]{
        \resizebox{.382\linewidth}{!}{
        \begin{tikzpicture}	
    		\begin{axis}[
                width=7cm,
        		height=6cm,
                xlabel=$\mu$ (logscale),
        		xtick={1,2,3,4,5,6,7,8,9,10},
                xticklabels={2,4,8,16,32,64,128,256,512,1024},
        		every axis y label/.style={rotate=0, black, at={(-0.13,0.5)},},
				]
				\addplot table[x=mu, y=w1]{\successrtsgnrnt};
                \addplot table[x=mu, y=w2]{\successrtsgnrnt};
                \addplot table[x=mu, y=w4]{\successrtsgnrnt};
                \addplot table[x=mu, y=w8]{\successrtsgnrnt};
                \addplot table[x=mu, y=w16]{\successrtsgnrnt};
                \addplot table[x=mu, y=w32]{\successrtsgnrnt};
                \addplot table[x=mu, y=w64]{\successrtsgnrnt};
                \addplot table[x=mu, y=w128]{\successrtsgnrnt};
    		\end{axis}
		\end{tikzpicture}
		}
		\label{fig:successrtsgnrnt}
    } 
    &
    \subfloat[Phenotypic]{
    \resizebox{.5\linewidth}{!}{
        \begin{tikzpicture}	
        	\begin{axis}[
        		width=7cm,
        		height=6cm,
        		xlabel=$\mu$ (logscale),
        		xtick={1,2,3,4,5,6,7,8,9,10},
                xticklabels={2,4,8,16,32,64,128,256,512,1024},
        		every axis y label/.style={rotate=0, black, at={(-0.13,0.5)},},
                legend columns=1,
                legend style={at={(1.35,1)},anchor=north east},
                legend cell align=left,
        	]
        		\addplot table[x=mu, y=w1]{\successrtspnrnt};
                \addplot table[x=mu, y=w2]{\successrtspnrnt};
                \addplot table[x=mu, y=w4]{\successrtspnrnt};
                \addplot table[x=mu, y=w8]{\successrtspnrnt};
                \addplot table[x=mu, y=w16]{\successrtspnrnt};
                \addplot table[x=mu, y=w32]{\successrtspnrnt};
                \addplot table[x=mu, y=w64]{\successrtspnrnt};
                \addplot table[x=mu, y=w128]{\successrtspnrnt};

                \legend{$w=1$, $w=2$, $w=4$, $w=8$, $w=16$, $w=32$, $w=64$, $w=128$, $w=256$, $w=512$, $w=1024$};
        	\end{axis}
		\end{tikzpicture}  	
		}
        \label{fig:successrtspnrnt}
    }
    \end{tabular}
    \caption{The number of successful runs measured among $100$ runs at the time both optima were found on \twomax or $t=100\max(1, \mu/w)\mu n \ln n$ generations have been reached for $n=100$ with the \muea with restricted tournament selection with and without replacement and $\mu\in\{2,4,8,\ldots,1024\}$, $w\in\{1,2,4,8,\ldots,128\}$, genotypic and phenotypic distance.}
    \label{fig:successrtsnewrunt}
\end{figure}
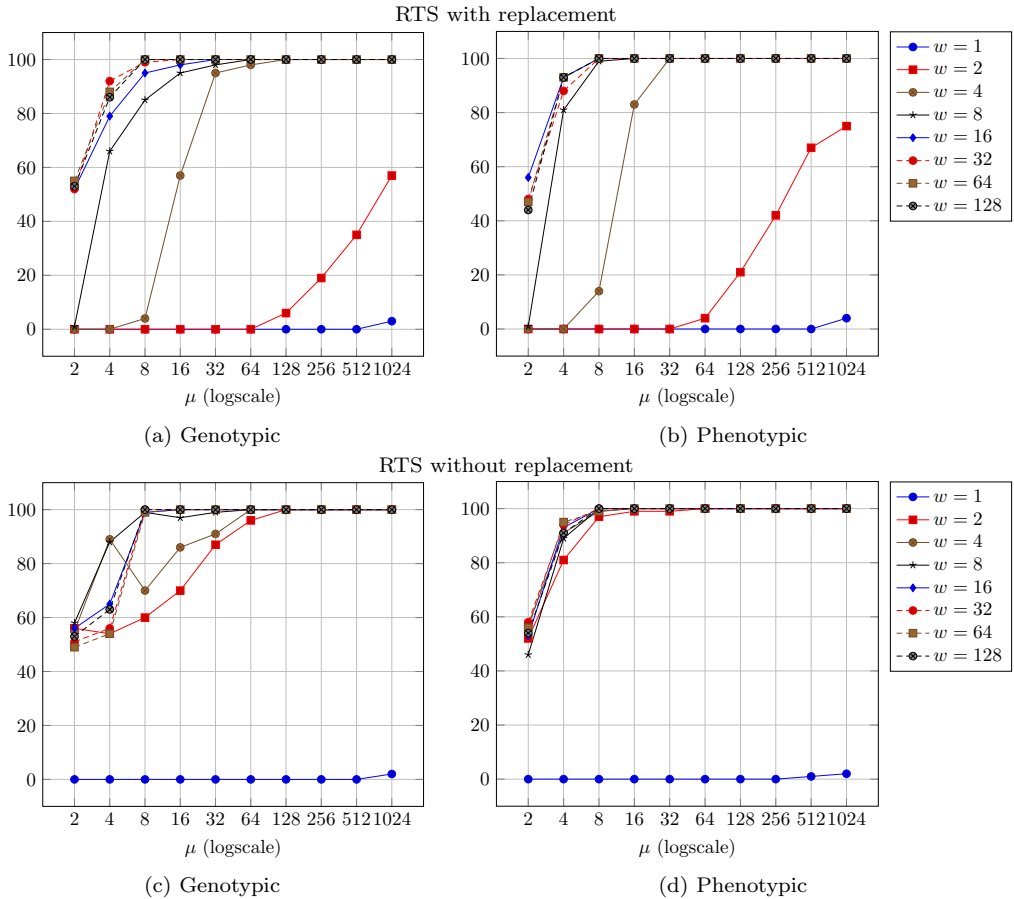

These results support our hypothesis that RTS without replacement has a better success rate than its variant with replacement for much smaller $w$ and $\mu$, but that this comes at the expense of a higher runtime as explained in Section~\ref{sec:rts-slowdown}. In this case takeover is not possible for $w=2$, but as can be seen from Figures~\ref{fig:successrtsgnrnt} and \ref{fig:successrtspnrnt}, evolving a single individual through the whole branch takes more time, which explains the sudden drop in the performance shown in  Figures~\ref{fig:successrtsgnr} and \ref{fig:successrtspnr}. This aligns with our hypothesis that the number of evaluations increases by a factor of $\mu^2/w$ to find both optima on \twomax. Finally, aside from the extreme case where $w=2$, the algorithm seems to behave similarly with both time budgets, achieving good results when $w\ge 8$ and $\mu\ge 8$ and achieving $100\%$ success rate when $w\ge 8$ and $\mu\ge 32$.
The choice $w=4$ gives mixed results as for both time budgets the success rate for RTS without replacement and genotypic distance is not monotonic in~$\mu$.

\subsubsection{Slow Down in the Performance of RTS Without Replacement on \twomax}
\label{sec:runtime_growth}

Now we look more closely into the slow down in the performance of RTS without replacement due to the appearance of the lone individual. We performed new experiments where we removed the time budget of generations, and we defined as the only stopping criterion either finding both optima, or that the population consists of copies of one optimum and with $w-1$ individuals with fitness $n-1$.

The reason for this slightly different stopping criterion is due to the selection without replacement of RTS. When the population has collapsed into one branch of \twomax, and the population has reached the optimum, once there are $w-1$ individuals in the population with fitness $n-1$, we claim that the algorithm is no longer able to replace the $w-1$ individuals with fitness $n-1$. Imagine that another copy of the present optimum is created. 
Since there are only $w-1$ non-optimal individuals,  any tournament of size $w$ must contain an optimum that is identical to the offspring. 
Then the new offspring will automatically compete with a global optimum since it is the closest individual in the tournament. In this sense, unless the opposite optimum is created by mutation, only replacements amongst optimal individuals are possible and the $w-1$ individuals with fitness $n-1$ will remain untouched. The algorithm will idle forever since it is not possible for all $\mu$ individuals to reach a fitness value of $n$. So once the algorithm reaches this ``stagnation'' scenario the run is stopped since there is no way to introduce new changes in the population.

In Figure~\ref{fig:runtime} we show how much time is needed by reporting the average number of generations achieved using the stopping criterion defined previously for the case of RTS without replacement with genotypic distance. An extra feature in this experimental setting is that we have introduced an additional curve for $w=\mu$ in order to observe whether the average runtime grows with $\mu^2/w$. Together with Figure~\ref{fig:runtime}, we provide in Table~\ref{tab:loneindapp} the mean and standard deviation of generations required for the same experimental setting as Figure~\ref{fig:runtime}, and the number of times the lone individual scenario came up among the 100 runs of the RTS without replacement.

\begin{figure}[!ht]
    \centering
    \resizebox{1\linewidth}{!}{
        \begin{tikzpicture}	
            \definecolor{oblue}{RGB}{0, 0, 255}
            \definecolor{fred}{RGB}{204, 0, 0}
            \definecolor{obrown}{RGB}{115, 77, 38}
            \definecolor{fbrown}{RGB}{153, 102, 51}
	        \begin{axis}[
                width=7cm,
        		height=6cm,
                xlabel=$\mu$ (logscale),
        		xtick={1,2,3,4,5,6,7,8,9,10},
                xticklabels={2,4,8,16,32,64,128,256,512,1024},
                ylabel=Average generations,
        		every axis y label/.style={rotate=90, black, at={(-0.15,0.5)},},
                legend columns=1,
                legend style={at={(2.52,1)},anchor=north east},
                legend cell align=left,
                name=ax1
		    ]
				\addplot table[x=mu, y=w2]{\rtsnrtwmu};
				\addplot table[x=mu, y=w4]{\rtsnrtwmu};
                \addplot table[x=mu, y=w8]{\rtsnrtwmu};
                \addplot table[x=mu, y=w16]{\rtsnrtwmu};
                \addplot table[x=mu, y=w32]{\rtsnrtwmu};
                \addplot table[x=mu, y=w64]{\rtsnrtwmu};
                \addplot table[x=mu, y=w128]{\rtsnrtwmu};
                \addplot table[x=mu, y=wmu]{\rtsnrtwmu};
                
                \coordinate (c1) at (axis cs:0.7,2000000);
                \coordinate (c2) at (axis cs:10.3,-2000000);
                \draw[dashed,line width = 1.5] (c1) rectangle (c2);
                
                \legend{$w=2$, $w=4$, $w=8$, $w=16$, $w=32$, $w=64$, $w=128$, $w=\mu$};
	        \end{axis}
	        \begin{axis}[
                width=7cm,
        		height=6cm,
                xlabel=$\mu$ (logscale),
        		xtick={1,2,3,4,5,6,7,8,9,10},
                xticklabels={2,4,8,16,32,64,128,256,512,1024},
                at={($(ax1.south east)+(1.2cm,0)$)},
                name=ax2
		    ]
				 \addplot+[mark=square*,mark options={fill=fred},red] table[x=mu, y=w4]{\rtsnrtwmu};
                 \addplot+[mark=otimes*,mark options={fill=fbrown},obrown] table[x=mu, y=w8]{\rtsnrtwmu};
                 \addplot+[mark=star,mark options={fill=black},black] table[x=mu, y=w16]{\rtsnrtwmu};
                 \addplot+[mark=diamond*,mark options={fill=blue},oblue] table[x=mu, y=w32]{\rtsnrtwmu};
                 \addplot+[mark=*,dashed,mark options={solid,fill=fred},red] table[x=mu, y=w64]{\rtsnrtwmu};
                 \addplot+[mark=square*,dashed,mark options={solid,fill=fbrown},obrown] table[x=mu, y=w128]{\rtsnrtwmu};
                \addplot+[mark=otimes*,dashed,mark options={solid,fill=gray},black] table[x=mu, y=wmu]{\rtsnrtwmu};
	        \end{axis}
	        \draw [dashed,line width = 1.5] (c1) -- (ax2.north west);
            \draw [dashed,line width = 1.5] (c2) -- (ax2.south west);
        \end{tikzpicture}
    }
    \caption{Average number of generations until either both optima of \twomax are found or the population consists of copies of the optimum and with $w-1$ individuals with fitness $n-1$ among 100 runs for $n=100$ with the \muea with restricted tournament selection without replacement for $\mu\in\{2,4,8,\ldots,1024\}$, $w\in\{2,4,8,\ldots,128,\mu\}$ and genotypic distance. The right-hand side shows a close-up view of the plot on the left-hand side (indicated by the dashed lines), without the curve for $w=2$.}
    \label{fig:runtime}
\end{figure}
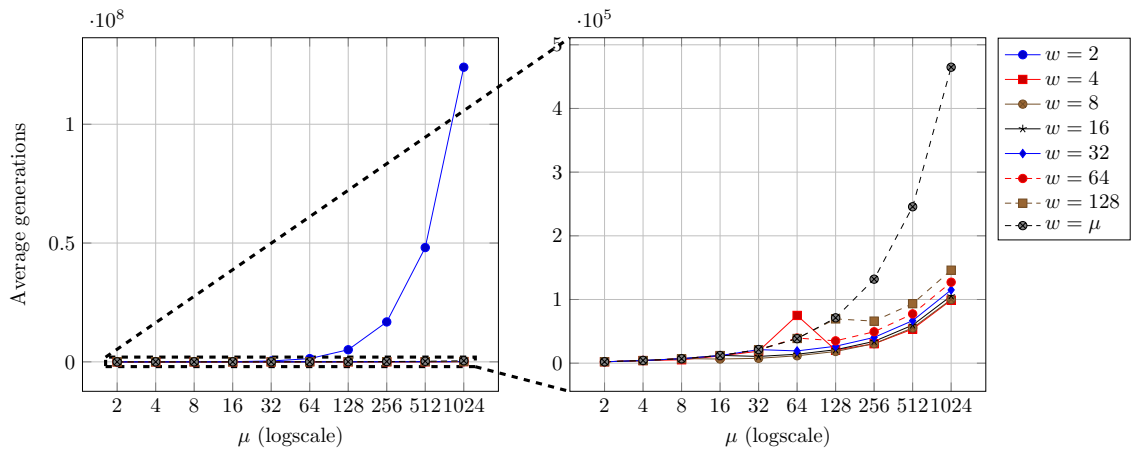

\begin{table}[!ht]
    \centering\footnotesize
    \caption{Mean and std of generations required to either find both optima of \twomax or the population consists of copies of the optimum and with $w-1$ individuals with fitness $n-1$, and the number of lone individual appearances among 100 runs for ${n=100}$ with the \muea with restricted tournament selection without replacement for $\mu\in\{2,4,8,\ldots,1024\}$, $w\in\{2,4,8,\ldots,128,\mu\}$ and genotypic distance.}
    \label{tab:loneindapp}
    \begin{tabular}{@{\quad}c@{\quad}cc@{\quad}c@{\quad}c@{\quad}c@{\quad}c@{\quad}c@{\quad}c@{\quad}c}
        \hline
        \multirow{2}{*}{\boldmath$\mu$} & & \multicolumn{8}{c}{\boldmath$w$} \\
        \cline{3-10}
         & & \bf 2 & \bf 4 & \bf 8 & \bf 16 & \bf 32 & \bf 64 & \bf 128 & \boldmath$\mu$ \\
        \hline
        \multirow{3}{*}{\boldmath 2} & mean & 2.22E+03 & 2.20E+03 & 2.07E+03 & 2.16E+03 & 2.22E+03 & 2.23E+03 & 2.20E+03 & 2.18E+03 \\
        & std & 6.31E+02 & 6.18E+02 & 5.71E+02 & 6.36E+02 & 6.45E+02 & 6.40E+02 & 5.36E+02 & 5.90E+02 \\
        & lone & 55 & 52 & 55 & 52 & 52 & 50 & 47 & 60 \\
        \midrule
        \multirow{3}{*}{\boldmath 4} & mean & 6.84E+03 & 4.07E+03 & 4.12E+03 & 4.15E+03 & 4.20E+03 & 4.21E+03 & 4.06E+03 & 4.28E+03 \\
        & std & 1.96E+03 & 9.05E+02 & 9.12E+02 & 1.12E+03 & 1.03E+03 & 8.09E+02 & 9.08E+02 & 9.81E+02 \\
        & lone & 88 & 64 & 68 & 60 & 57 & 71 & 65 & 66 \\
        \midrule\multirow{3}{*}{\boldmath 8} & mean & 2.35E+04 & 5.28E+03 & 7.15E+03 & 7.00E+03 & 7.08E+03 & 7.17E+03 & 6.95E+03 & 6.90E+03 \\
        & std & 6.23E+03 & 3.11E+03 & 1.39E+03 & 1.65E+03 & 1.49E+03 & 1.35E+03 & 1.32E+03 & 1.65E+03 \\
        & lone & 97 & 44 & 10 & 6 & 6 & 8 & 14 & 7 \\
        \midrule\multirow{3}{*}{\boldmath 16} & mean & 9.02E+04 & 1.13E+04 & 6.50E+03 & 1.22E+04 & 1.19E+04 & 1.20E+04 & 1.20E+04 & 1.19E+04 \\
        & std & 2.58E+04 & 1.09E+04 & 7.29E+03 & 1.57E+03 & 1.65E+03 & 1.61E+03 & 1.84E+03 & 1.69E+03 \\
        & lone & 100 & 21 & 4 & 0 & 1 & 0 & 0 & 0 \\
        \midrule\multirow{3}{*}{\boldmath 32} & mean & 3.67E+05 & 1.89E+04 & 7.62E+03 & 1.06E+04 & 2.13E+04 & 2.13E+04 & 2.08E+04 & 2.14E+04 \\
        & std & 1.10E+05 & 3.39E+04 & 1.12E+03 & 1.53E+03 & 2.39E+03 & 2.38E+03 & 2.27E+03 & 2.36E+03 \\
        & lone & 100 & 7 & 1 & 0 & 0 & 0 & 0 & 0 \\
        \midrule\multirow{3}{*}{\boldmath 64} & mean & 1.37E+06 & 7.50E+04 & 1.16E+04 & 1.41E+04 & 1.93E+04 & 3.95E+04 & 3.88E+04 & 3.84E+04 \\
        & std & 4.14E+05 & 5.98E+05 & 2.39E+03 & 1.71E+03 & 2.51E+03 & 4.03E+03 & 4.43E+03 & 3.97E+03 \\
        & lone & 100 & 1 & 0 & 0 & 0 & 0 & 0 & 0 \\
        \midrule\multirow{3}{*}{\boldmath 128} & mean & 5.10E+06 & 2.05E+04 & 1.83E+04 & 2.11E+04 & 2.66E+04 & 3.51E+04 & 6.95E+04 & 7.12E+04 \\
        & std & 1.95E+06 & 5.42E+03 & 1.44E+03 & 1.94E+03 & 2.44E+03 & 3.71E+03 & 6.70E+03 & 6.01E+03 \\
        & lone & 96 & 0 & 0 & 0 & 0 & 0 & 0 & 0 \\
        \midrule\multirow{3}{*}{\boldmath 256} & mean & 1.68E+07 & 3.05E+04 & 3.07E+04 & 3.42E+04 & 4.05E+04 & 4.93E+04 & 6.60E+04 & 1.32E+05 \\
        & std & 8.79E+06 & 4.99E+03 & 2.05E+03 & 2.41E+03 & 3.48E+03 & 3.59E+03 & 6.58E+03 & 1.01E+04 \\
        & lone & 88 & 0 & 0 & 0 & 0 & 0 & 0 & 0 \\
        \midrule\multirow{3}{*}{\boldmath 512} & mean & 4.81E+07 & 5.35E+04 & 5.54E+04 & 5.97E+04 & 6.68E+04 & 7.74E+04 & 9.34E+04 & 2.46E+05 \\
        & std & 3.89E+07 & 5.26E+03 & 3.15E+03 & 3.98E+03 & 4.82E+03 & 5.51E+03 & 7.22E+03 & 1.80E+04 \\
        & lone & 70 & 0 & 0 & 0 & 0 & 0 & 0 & 0 \\
        \midrule\multirow{3}{*}{\boldmath 1024} & mean & 1.24E+08 & 9.90E+04 & 1.00E+05 & 1.06E+05 & 1.15E+05 & 1.27E+05 & 1.46E+05 & 4.65E+05 \\
        & std & 1.42E+08 & 8.76E+03 & 5.56E+03 & 5.64E+03 & 6.86E+03 & 8.32E+03 & 1.00E+04 & 3.64E+04 \\
        & lone & 50 & 0 & 0 & 0 & 0 & 0 & 0 & 0 \\
        \bottomrule
    \end{tabular}
\end{table}

First of all, from the raw data we observed that the increased time budget of $t:=100\allowbreak\max(1, \mu/w)\mu n \ln n$ was never reached, hence it seems to be large enough, and effectively the RTS without replacement needs more time to find both optima on \twomax. From Figure~\ref{fig:runtime}, when $w=2$, we can observe that the average runtime immediately increases as soon as the population size increases. Note that in Table~\ref{tab:loneindapp} for $w=2$ lone individuals appear frequently. This supports our hypothesis from Section~\ref{sec:rts-slowdown} that due to the lone individual appearance during the run, there is an increase in the mean of generations needed to find their respective optimum. The mean increases when increasing the population size~$\mu$ and the standard deviation increases in line with the mean. 
In most runs, the algorithm is idling as the lone individual needs to be selected for both selection for reproduction and selection for the tournament.


In the case of $4\leq w \leq \mu$ we can observe an improvement in the performance of the algorithm. The chances of spending time idling for a significant replacement decrease since we are allowing more individuals to participate in the tournament. From both, Figure~\ref{fig:runtime} and Table~\ref{tab:loneindapp}, for $w=4$ we still can see some of this idling behaviour with fewer appearances of the lone individual as the population size increases. However, the means and standard deviations are generally smaller than for $w=2$. For the case of $8 \leq w \leq 128$ the algorithm seems to be much faster than predicted by our conjectured bound $\Bigom{(\mu^2/w) \cdot n \log n}$. A plausible explanation drawn from Table~\ref{tab:loneindapp} is that for $w \ge 8$ the chances of lone individuals emerging are much lower than for $w \in \{2, 4\}$.


Finally, for the case of $w=\mu$ there is a slow down towards the optimum because the whole population is part of the tournament, and only replacements can be done with individuals close to each other, which translates into small jumps and more time needed to move towards the optimum. In this case the effect of RTS vanishes and the size of the population is the only factor in the growth of the time, something that we were expecting from the term $\mu^2/w$ in our conjectured bound of $\Bigom{(\mu^2/w) \cdot n \log n}$. 

It seems that our conjectured bound of $\Bigom{(\mu^2/w) \cdot n \log n}$ only applies to very small values of $w$ due to the appearance of the lone individual and the necessity of specific selections in one step. For the case of large values of $w=\mu$, the RTS effect vanishes and we end up with an algorithm similar to deterministic crowding in which the only main factor on the runtime is the population size. For intermediate values of $w$, it seems that a more relaxed time bound can be applied since unwanted replacements are more difficult (more individuals participate in the tournament) and specific selections for replacement are not needed (large jumps in the same branch are possible) but this may need more detailed conditions or arguments related to the population dynamics and distribution on the population.

\subsection{Progress on the Inferior Niche}
\label{sec:exp_takeover}

The main goal of any niching mechanism is to avoid (or at least delay) that the fittest individuals take over the less fit individuals, which ultimately leads to a loss of diversity. From the previous theoretical and empirical analysis we can observe that RTS is able to delay takeover depending on the selection scheme, and the parameters $w$ and $\mu$. We know that RTS with replacement, for small values of~$w$ and $\mu$, is not able to maintain individuals on both branches of \twomax for a long period of time. For the case of RTS without replacement we know that takeover it is more difficult to happen, the algorithm is more resilient to takeover, but as we mentioned in Section~\ref{sec:small_w_rts_nor}, takeover is still possible.

So, when exactly is takeover more likely to happen? When exactly during the process does a subpopulation become extinct? When and under what situations does extinction happen? These are the kind of questions we aim to answer in this section. In order to observe when takeover is more likely to happen with both variants of RTS, we have designed the following experimental analysis. We recorded the fitness of the best individuals reached on both branches of \twomax, and then take the minimum fitness value of those two individuals. This yields the maximum fitness reached by the subpopulation that becomes extinct (if applicable), or the optimal fitness if both optima were reached. These results are shown in Figure~\ref{fig:takeover}. We define the specifics of the experimental setup in its caption. 

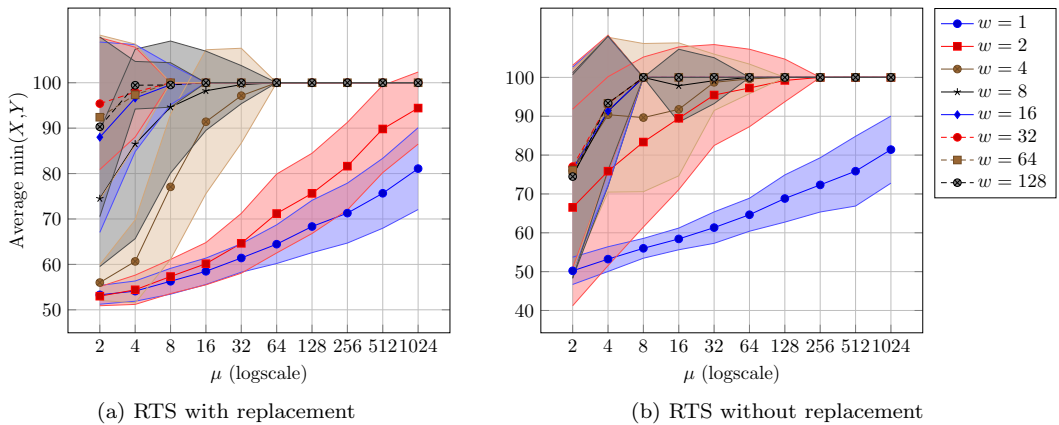
\begin{figure}[!ht]
    \centering
        \subfloat[RTS with replacement]{
            \resizebox{.404\linewidth}{!}{
		        \begin{tikzpicture}	
    		        \begin{axis}[
                        width=7cm,
        		        height=6cm,
                        xlabel=$\mu$ (logscale),
        		        xtick={1,2,3,4,5,6,7,8,9,10},
                        xticklabels={2,4,8,16,32,64,128,256,512,1024},
                        ylabel=Average $\min(X\text{,}Y)$,
                        ytick={40,50,60,70,80,90,100},
        		        every axis y label/.style={rotate=90, black, at={(-0.13,0.5)},},
				    ]
    				    \addplot table[x=mu, y=w1_m]{\wrgmeanstd};
                        \addplot table[x=mu, y=w2_m]{\wrgmeanstd};
                        \addplot table[x=mu, y=w4_m]{\wrgmeanstd};
                        \addplot table[x=mu, y=w8_m]{\wrgmeanstd};
                        \addplot table[x=mu, y=w16_m]{\wrgmeanstd};
                        \addplot table[x=mu, y=w32_m]{\wrgmeanstd};
                        \addplot table[x=mu, y=w64_m]{\wrgmeanstd};
                        \addplot table[x=mu, y=w128_m]{\wrgmeanstd};
                        
                        \addplot [name path=upperw1,color=blue!70] table[x=mu,y expr=\thisrow{w1_m}+\thisrow{w1_s}] {\wrgmeanstd};
                        \addplot [name path=lowerw1,color=blue!70] table[x=mu,y expr=\thisrow{w1_m}-\thisrow{w1_s}] {\wrgmeanstd};
                        \addplot [blue!50,fill opacity=0.5] fill between[of=upperw1 and lowerw1];
                        
                        \addplot [name path=upperw2,color=red!70] table[x=mu,y expr=\thisrow{w2_m}+\thisrow{w2_s}] {\wrgmeanstd};
                        \addplot [name path=lowerw2,color=red!70] table[x=mu,y expr=\thisrow{w2_m}-\thisrow{w2_s}] {\wrgmeanstd};
                        \addplot [red!50,fill opacity=0.5] fill between[of=upperw2 and lowerw2];
                        
                        \addplot [name path=upperw4,color=brown!70] table[x=mu,y expr=\thisrow{w4_m}+\thisrow{w4_s}] {\wrgmeanstd};
                        \addplot [name path=lowerw4,color=brown!70] table[x=mu,y expr=\thisrow{w4_m}-\thisrow{w4_s}] {\wrgmeanstd};
                        \addplot [brown!50,fill opacity=0.5] fill between[of=upperw4 and lowerw4];
                        
                        \addplot [name path=upperw8,color=black!70] table[x=mu,y expr=\thisrow{w8_m}+\thisrow{w8_s}] {\wrgmeanstd};
                        \addplot [name path=lowerw8,color=black!70] table[x=mu,y expr=\thisrow{w8_m}-\thisrow{w8_s}] {\wrgmeanstd};
                        \addplot [black!50,fill opacity=0.5] fill between[of=upperw8 and lowerw8];
                        
                        \addplot [name path=upperw16,color=blue!70] table[x=mu,y expr=\thisrow{w16_m}+\thisrow{w16_s}] {\wrgmeanstd};
                        \addplot [name path=lowerw16,color=blue!70] table[x=mu,y expr=\thisrow{w16_m}-\thisrow{w16_s}] {\wrgmeanstd};
                        \addplot [blue!50,fill opacity=0.5] fill between[of=upperw16 and lowerw16];
                        
                        \addplot [name path=upperw32,color=red!70] table[x=mu,y expr=\thisrow{w32_m}+\thisrow{w32_s}] {\wrgmeanstd};
                        \addplot [name path=lowerw32,color=red!70] table[x=mu,y expr=\thisrow{w32_m}-\thisrow{w32_s}] {\wrgmeanstd};
                        \addplot [red!50,fill opacity=0.5] fill between[of=upperw32 and lowerw32];
                        
                        \addplot [name path=upperw64,color=brown!70] table[x=mu,y expr=\thisrow{w64_m}+\thisrow{w64_s}] {\wrgmeanstd};
                        \addplot [name path=lowerw64,color=brown!70] table[x=mu,y expr=\thisrow{w64_m}-\thisrow{w64_s}] {\wrgmeanstd};
                        \addplot [brown!50,fill opacity=0.5] fill between[of=upperw64 and lowerw64];
                        
                        \addplot [name path=upperw128,color=black!70] table[x=mu,y expr=\thisrow{w128_m}+\thisrow{w128_s}] {\wrgmeanstd};
                        \addplot [name path=lowerw128,color=black!70] table[x=mu,y expr=\thisrow{w128_m}-\thisrow{w128_s}] {\wrgmeanstd};
                        \addplot [black!50,fill opacity=0.5] fill between[of=upperw128 and lowerw128];
    		        \end{axis}
		        \end{tikzpicture}
            }
            \label{fig:takeoverbestrts}
        }
        ~
        \subfloat[RTS without replacement]{
            \resizebox{.5\linewidth}{!}{
		        \begin{tikzpicture}	
    		        \begin{axis}[
                        width=7cm,
        		        height=6cm,
                        xlabel=$\mu$ (logscale),
                		xtick={1,2,3,4,5,6,7,8,9,10},
                        xticklabels={2,4,8,16,32,64,128,256,512,1024},
                        ytick={40,50,60,70,80,90,100},
                		every axis y label/.style={rotate=90, black, at={(-0.13,0.5)},},
                		legend columns=1,
                        legend style={at={(1.35,1)},anchor=north east},
                        legend cell align=left,
				    ]
        				\addplot table[x=mu, y=w1_m]{\wtrgmeanstd};
                        \addplot table[x=mu, y=w2_m]{\wtrgmeanstd};
                        \addplot table[x=mu, y=w4_m]{\wtrgmeanstd};
                        \addplot table[x=mu, y=w8_m]{\wtrgmeanstd};
                        \addplot table[x=mu, y=w16_m]{\wtrgmeanstd};
                        \addplot table[x=mu, y=w32_m]{\wtrgmeanstd};
                        \addplot table[x=mu, y=w64_m]{\wtrgmeanstd};
                        \addplot table[x=mu, y=w128_m]{\wtrgmeanstd};
                        
                        \addplot [name path=upperw1,color=blue!70] table[x=mu,y expr=\thisrow{w1_m}+\thisrow{w1_s}] {\wtrgmeanstd};
                        \addplot [name path=lowerw1,color=blue!70] table[x=mu,y expr=\thisrow{w1_m}-\thisrow{w1_s}] {\wtrgmeanstd};
                        \addplot [blue!50,fill opacity=0.5] fill between[of=upperw1 and lowerw1];
                        
                        \addplot [name path=upperw2,color=red!70] table[x=mu,y expr=\thisrow{w2_m}+\thisrow{w2_s}] {\wtrgmeanstd};
                        \addplot [name path=lowerw2,color=red!70] table[x=mu,y expr=\thisrow{w2_m}-\thisrow{w2_s}] {\wtrgmeanstd};
                        \addplot [red!50,fill opacity=0.5] fill between[of=upperw2 and lowerw2];
                        
                        \addplot [name path=upperw4,color=brown!70] table[x=mu,y expr=\thisrow{w4_m}+\thisrow{w4_s}] {\wtrgmeanstd};
                        \addplot [name path=lowerw4,color=brown!70] table[x=mu,y expr=\thisrow{w4_m}-\thisrow{w4_s}] {\wtrgmeanstd};
                        \addplot [brown!50,fill opacity=0.5] fill between[of=upperw4 and lowerw4];
                        
                        \addplot [name path=upperw8,color=black!70] table[x=mu,y expr=\thisrow{w8_m}+\thisrow{w8_s}] {\wtrgmeanstd};
                        \addplot [name path=lowerw8,color=black!70] table[x=mu,y expr=\thisrow{w8_m}-\thisrow{w8_s}] {\wtrgmeanstd};
                        \addplot [black!50,fill opacity=0.5] fill between[of=upperw8 and lowerw8];
                        
                        \addplot [name path=upperw16,color=blue!70] table[x=mu,y expr=\thisrow{w16_m}+\thisrow{w16_s}] {\wtrgmeanstd};
                        \addplot [name path=lowerw16,color=blue!70] table[x=mu,y expr=\thisrow{w16_m}-\thisrow{w16_s}] {\wtrgmeanstd};
                        \addplot [blue!50,fill opacity=0.5] fill between[of=upperw16 and lowerw16];
                        
                        \addplot [name path=upperw32,color=red!70] table[x=mu,y expr=\thisrow{w32_m}+\thisrow{w32_s}] {\wtrgmeanstd};
                        \addplot [name path=lowerw32,color=red!70] table[x=mu,y expr=\thisrow{w32_m}-\thisrow{w32_s}] {\wtrgmeanstd};
                        \addplot [red!50,fill opacity=0.5] fill between[of=upperw32 and lowerw32];
                        
                        \addplot [name path=upperw64,color=brown!70] table[x=mu,y expr=\thisrow{w64_m}+\thisrow{w64_s}] {\wtrgmeanstd};
                        \addplot [name path=lowerw64,color=brown!70] table[x=mu,y expr=\thisrow{w64_m}-\thisrow{w64_s}] {\wtrgmeanstd};
                        \addplot [brown!50,fill opacity=0.5] fill between[of=upperw64 and lowerw64];
                        
                        \addplot [name path=upperw128,color=black!70] table[x=mu,y expr=\thisrow{w128_m}+\thisrow{w128_s}] {\wtrgmeanstd};
                        \addplot [name path=lowerw128,color=black!70] table[x=mu,y expr=\thisrow{w128_m}-\thisrow{w128_s}] {\wtrgmeanstd};
                        \addplot [black!50,fill opacity=0.5] fill between[of=upperw128 and lowerw128];
                        \legend{$w=1$, $w=2$, $w=4$, $w=8$, $w=16$, $w=32$, $w=64$, $w=128$};
    		        \end{axis}
		        \end{tikzpicture}
            }
            \label{fig:takeoverbestrtsnr}
        }
    \caption{Average (pointed line) and Standard Deviation (shaded area) of minimum fitness reached between the best individual~$X$ found from the $0^n$ branch the best individual~$Y$ found from the $1^n$ branch ($\min(X,Y)$) measured among 100 runs at the time both optima were found on \twomax, the population has collapsed into one optimum on \twomax (the population consists of copies of just $0^n$ or $1^n$) with no time budget for $n=100$ with the \muea with restricted tournament selection with and without replacement for $\mu\in\{2,4,8,\ldots,1024\}$, $w\in\{1,2,4,8,\ldots,128\}$ and genotypic distance.}
    \label{fig:takeover}
\end{figure}

For $w=1$, both variants of RTS do not show major differences. The major difference in the performance starts when $w\ge 2$ for all populations sizes. RTS without replacement is able to avoid the extinction of one subpopulation for smaller population sizes than its variant with replacement, \ie while RTS without replacement is able to maintain individuals with fitness values up to~$65$ for small population sizes ($\mu=2$) and small window sizes (${w\ge 2}$), its variant with replacement requires greater window sizes ($w\ge 8$) for the same population size to achieve similar results. In general, both RTS variants have a better performance for large $\mu$ and for large $w$, which allows more individuals to participate in the tournaments. The main difference between both RTS variants is that RTS with replacement is not able to maintain populations on both branches of \twomax for a long time after initialisation for small $\mu$ and $w$. Takeover happens when the individuals have a fitness of around $50$, which indicates that takeover happens close to initialisation. This makes sense since individuals on different branches may look similar to each other and with small $w$, unwanted replacements may happen. 

RTS without replacement is more resistant to extinction after initialisation for small $\mu$ and $w$, but requires more fitness evaluations. This can be observed simply from the results on Figure~\ref{fig:takeover} but specifically if we compare them with $w\ge 2$ and $\mu\ge 2$, RTS without replacement is able to reach fitness values above $65$ compared to its variant with replacement.

\section{Conclusion}
\label{sec:con}

We theoretically and empirically examined the behaviour of restricted tournament selection, embedded into a simple \muea, on the bimodal function \twomax, where the goal is to find both optima. We rigorously proved that the performance of RTS varies a lot with the window size~$w$. If $w$ is large enough, $w \ge 2.5\mu \ln n$, then RTS behaves similarly to deterministic crowding. The probability of finding both optima in time $\Bigo{\mu n \log n}$ is close to $1-2^{-\mu+1}$, hence converging to~1 very quickly as $\mu$ grows. For small $\mu$ and $w$, if $\mu \le n^{1/(w-1)}$ (and $w \ge 3$) then RTS likely fails to find both optima of \twomax. This even holds when the \muea is allowed to start with the most promising population that does not yet contain both optima.

When selecting individuals for the tournament \emph{without} replacement, the tournament becomes more diverse. While the positive result for the original RTS easily transfers with the more lax condition~$w \ge \mu$, the population dynamics in the case of small~$w$ become more complicated.
Experiments suggest that for small values of~$w$ typically one niche collapses to a single individual, and we proved that under certain conditions the algorithm is very slow at climbing up said branch. We conjecture a lower time bound of $\Bigom{(\mu^2/w) \cdot  n \log n}$ for small values of~$w$, where it is common for a single ``lone'' individual to evolve. This lower bound is by a factor of $\mu/w$ larger than the upper time bound for RTS with large~$w$ (Theorems~\ref{the:positive-result-for-rts} and~\ref{the:positive-result-for-rts-without-replacement}).
Experiments support this conjecture as for $n=100$ the runtime increases drastically when $w$ is very small. Table~\ref{tab:loneindapp} showed that the average runtime is high when lone individuals emerged, and lone individuals typically emerged for very small~$w$ and small~$\mu$.

Our theoretical results cover small and large values for the window size~$w$. It is still an open problem to theoretically analyse the population dynamics for both RTS variants for intermediate values for~$w$. Our experiments indicate that RTS with and without replacement can optimise \twomax for smaller~$w$ than those required in Theorems~\ref{the:positive-result-for-rts} and~\ref{the:positive-result-for-rts-without-replacement}.
Showing refined upper bounds for smaller values of~$w$ that make this rigorous remains an open problem. Likewise, showing improved lower bounds for larger values of~$w$ than those given in Theorem~\ref{the:badresrts}, or proving the conjectured lower bound of $\Omega((\mu^2/w) n \log n)$ for RTS without replacement remain open problems. Note that proving lower bounds for population-based algorithms is a notoriously hard challenge even in the absence of diversity mechanisms, which may require new methods to be developed (see, e.\,g.~\cite{SuttonWittGECCO19,Oliveto2020} for recent approaches in this direction).


\section*{Acknowledgments}

The authors would like to thank the Consejo Nacional de Ciencia y Tecnolog\'{i}a --- CONACYT (the Mexican National Council for Science and Technology) for the financial support under the grant no.\ 409151 and registration no.\ 264342.

\bibliographystyle{apalike}
\bibliography{bibliography}

\end{document}